\documentclass{article}

 \usepackage[preprint,nonatbib]{neurips_2019}

\usepackage{microtype}
\usepackage{graphicx}
\usepackage{algorithmic, algorithm}

 \usepackage{booktabs} 
\usepackage{caption}
\usepackage{subfig}
\usepackage{nicefrac}

\usepackage[colorinlistoftodos,bordercolor=orange,backgroundcolor=orange!20,linecolor=orange,textsize=scriptsize]{todonotes}

\newcommand{\R}{\mathbb{R}}
\newcommand{\dotprod}[2]{\left< #1,#2\right>}
\newcommand{\norm}[1]{\left\lVert#1\right\rVert}

\newcommand{\eqdef}{\overset{\text{def}}{=}}

\newcommand{\Exp}[1]{{{\rm E}}\left[#1\right] }    
\newcommand{\E}[1]{{\rm E}\left[#1\right] }

\newcommand{\cD}{{\cal D}}

\newcommand{\cO}{{\cal O}}

\newcommand{\cW}{{\cal W}}
\newcommand{\cX}{{\cal X}}
\newcommand{\cY}{{\cal Y}}
\newcommand{\cZ}{{\cal Z}}

\newcommand{\SM}{Loopless SVRG (\texttt{L-SVRG})}
\newcommand{\ASM}{Loopless Katyusha (\texttt{L-Katyusha})}

\newcommand{\SMs}{\texttt{L-SVRG}}
\newcommand{\ASMs}{\texttt{L-Katyusha}}

\usepackage{amsmath,amsfonts,amssymb,amsthm,array}
\usepackage{hyperref}
\usepackage[capitalize,noabbrev]{cleveref}

\usepackage{mdframed} 
\usepackage{thmtools}

\definecolor{shadecolor}{gray}{1.00}
\declaretheoremstyle[
headfont=\normalfont\bfseries,
notefont=\mdseries, notebraces={(}{)},
bodyfont=\normalfont,
postheadspace=0.5em,
spaceabove=1pt,
mdframed={
	skipabove=5pt, 
	skipbelow=5pt, 
	hidealllines=true,
	backgroundcolor={shadecolor},
	innerleftmargin=4pt,
	innerrightmargin=4pt}
]{shaded}

\declaretheorem[style=shaded,within=section]{definition}
\declaretheorem[style=shaded,sibling=definition]{theorem}

\declaretheorem[style=shaded,sibling=definition]{assumption}

\declaretheorem[style=shaded,sibling=definition]{lemma}

\title{Don't Jump Through Hoops and Remove Those Loops:  SVRG and Katyusha are Better Without the Outer Loop}

\author{%
  Dmitry Kovalev \\
  KAUST, KSA \\
  \texttt{dmitry.kovalev@kaust.edu.sa} \\
  \And
  Samuel Horv\'ath \\
  KAUST, KSA \\
  \texttt{samuel.horvath@kaust.edu.sa} \\
  \And
  Peter Richt\'arik \\
  KAUST, KSA  and MIPT, Russia\\
  \texttt{peter.richtarik@kaust.edu.sa} \\
}

\begin{document}

\maketitle

\begin{abstract}
The stochastic variance-reduced gradient method (\texttt{SVRG}) and its accelerated variant (\texttt{Katyusha}) have attracted enormous attention in the machine learning community in the last few years due to their superior theoretical properties and empirical behaviour on training supervised machine learning models via the empirical risk minimization paradigm. A key structural element in both of these methods is the inclusion of an outer loop at the beginning of which  a full pass over the training data is made in order to compute the exact gradient, which is then used in an inner loop to construct a variance-reduced estimator of the gradient using new stochastic gradient information. In this work we design {\em loopless variants} of both of these methods. In particular,  we remove the outer loop and replace its function by a coin flip performed in each iteration designed to trigger, with a small probability, the computation of the gradient. We prove that the new methods enjoy the same superior theoretical convergence properties as the original methods. For loopless \texttt{SVRG},  the same rate is obtained for a large interval of coin flip probabilities, including the probability $\nicefrac{1}{n}$, where $n$ is the number of functions. This is the first result where a variant of \texttt{SVRG} is shown to converge with the same rate without the need for the algorithm to know the condition number, which is often unknown or hard to estimate correctly. We demonstrate through numerical experiments that the loopless methods can have  superior and more robust practical behavior.
\end{abstract}

\section{Introduction}

Empirical risk minimization (aka finite-sum) problems form the dominant paradigm for training supervised machine learning models such as ridge regression, support vector machines, logistic regression, and neural networks. In its most general form, a finite sum problem has the form
\begin{equation}
\label{eq:problem}
\textstyle \min_{x \in \R^d}  f(x) \eqdef \tfrac{1}{n}\sum \limits_{i=1}^n f_i(x)  \;,
\end{equation}
where $n$ refers to the number of training data points (e.g., videos, images, molecules), $x$ is the vector representation of a model using $d$ features, and $f_i(x)$ is the loss of model $x$ on data point $i$.

{\bf Variance-reduced methods.} One of the most remarkable algorithmic breakthroughs in recent years was the development of {\em variance-reduced} stochastic gradient algorithms for solving \eqref{eq:problem}. These methods are significantly faster than \texttt{SGD} \cite{nemirovsky1983problem, Nemirovski_2009, pegasos2} in theory and practice on convex and strongly convex problems, and faster in theory on several classes on nonconvex problems (unfortunately, these methods are no yet successful in training production-grade neural networks).

Two of the most notable and popular methods belonging to the family of variance-reduced methods are \texttt{SVRG}
 \cite{johnson2013accelerating} and its accelerated variant known as \texttt{Katyusha} \cite{allen2017katyusha}. The latter method accelerates the former via the employment of a novel ``negative momentum'' idea. Both of these methods have a double loop design. At the beginning of the outer loop, a full pass over the training data is made to compute the gradient of $f$ at a reference point $w^k$, which is chosen as the freshest iterate  (\texttt{SVRG}) or a weighted average of recent iterates (for \texttt{Katyusha}). This gradient is then used in the inner loop to {\em adjust} the stochastic gradient $\nabla f_i(x^k)$, where $i$ is sampled uniformly at random from $[n]\eqdef \{1,2,\dots,n\}$, and $x^k$ is the current iterate, so as to reduce its variance. In particular, both \texttt{SVRG} and \texttt{Katyusha} perform the adjustment
$g^k =  \nabla f_i(x^k) - \nabla f_i(w^k) + \nabla f(w^k).$
Note that, like $\nabla f_i(x^k)$, the new search direction $g^k$ is an unbiased estimator of $\nabla f(x^k)$. Indeed,
	\begin{equation}\label{eq:unbiased}
		\E{g^k} = \nabla f(x^k) - \nabla f(w^k) + \nabla f(w^k) =  \nabla f(x^k).
	\end{equation}
	where the expectation is taken over random choice of $i \in [n]$. However, it turns out that as the methods progress, the variance of $g^k$, unlike that of $\nabla f_i(x^k)$, progressively decreases to zero. The total effect of this is significantly faster convergence. 

{\bf Converegnce of \texttt{SVRG} and \texttt{Katyusha} for $L$--smooth and $\mu$--strongly convex functions.} For instance, consider the regime where $f_i$ is $L$--smooth for each $i$, and $f$ is $\mu$--strongly convex:

\begin{assumption}[$L$--smoothness]
	Functions $f_i: \R^d \rightarrow \R$ are $L$--smooth for some $L > 0$:
	\begin{equation}
	\label{def:L-smoothness}
	f(y) \leq f(x) + \dotprod{\nabla f(x)}{y-x} +\tfrac{L}{2}\norm{y-x}^2, \qquad \forall x,y \in \R^d.
	\end{equation}
\end{assumption}

\begin{assumption}[$\mu$--strong convexity]
	Function $f: \R^d \rightarrow \R$ is $\mu$--strongly convex for $\mu > 0$:
	\begin{equation}
	\label{def:strong_convexity}
		f(y) \geq f(x) + \dotprod{\nabla f(x)}{y-x} +\tfrac{\mu}{2}\norm{y-x}^2, \qquad \forall x,y \in \R^d.
	\end{equation}
\end{assumption}

 In this regime, the iteration complexity of \texttt{SVRG} is $\cO\left(\left(n + \nicefrac{L}{\mu}\right)\log \nicefrac{1}{\epsilon}\right),$ which is a vast improvement on the linear rate of gradient descent (\texttt{GD}), which is $\cO\left( \nicefrac{ nL}{\mu} \log \nicefrac{1}{\epsilon} \right)$, and on the sublinear rate of \texttt{SGD}, which is $\cO(\nicefrac{L}{\mu} + \nicefrac{\sigma^2}{\mu^2 \epsilon})$, where $\sigma^2=\nicefrac{1}{n}\sum_i \|\nabla f_i(x^*)\|^2$ and $x^*$ is the (necessarily unique) minimizer of $f$. On the other hand, \texttt{Katyusha} enjoys the {\em accelerated} rate  $\cO((n + \sqrt{\nicefrac{nL}{\mu}} )\log \nicefrac{1}{\epsilon} ),$ which is  superior to that of \texttt{SVRG} in the ill-conditioned regime where $\nicefrac{L}{\mu}\geq n$.  This rate has been shown to be {\em optimal} in a certain precise sense \cite{nesterov2013introductory}.

In the past several years, an enormous effort of the machine learning and optimization communities was exerted into designing new efficient variance-reduced methods algorithms to tackle problem \eqref{eq:problem}. These developments have brought about a renaissance in the field. The historically first provably variance-reduced method, the stochastic average gradient (\texttt{SAG}) method of \cite{SAG, schmidt2017minimizing}, was awarded the Lagrange prize in continuous optimization in 2018. The \texttt{SAG} method was later modified to an unbiased variant called \texttt{SAGA}~\cite{SAGA}, achieving the same theoretical rates.  Alternative variance-reduced method  include  \texttt{MISO}~\cite{MISO}, \texttt{FINITO}~\cite{FINITO}, \texttt{SDCA}~\cite{shalev2016sdca}, \texttt{dfSDCA}~\cite{dfSDCA}, \texttt{AdaSDCA}~\cite{ADASDCA},  \texttt{QUARTZ}~\cite{Quartz}, \texttt{SBFGS}~\cite{SBFGS}, \texttt{SDNA}~\cite{SDNA}, \texttt{SARAH}~\cite{nguyen2017sarah} and  \texttt{S2GD}~\cite{S2GD}, \texttt{mS2GD}~\cite{mS2GD}, \texttt{RBCN}~\cite{RBCN}, \texttt{JacSketch}~\cite{JacSketch} and \texttt{SAGD}~\cite{SAGD}. Accelerated variance-reduced method were developed in \cite{shalev2014accelerated}, \cite{defazio2016simple}, \cite{zhou2018direct} and \cite{zhou2018simple}.

\section{Contributions}

As explained in the introduction, a trade-mark structural feature of \texttt{SVRG} and its accelerated variant, \texttt{Katyusha}, is the presence of the outer loop in which a full pass over the data is made. However, the presence of this outer loop is the source of several issues. First, the methods are harder to analyze. Second, one needs to decide at which point the inner loop is terminated and the outer loop entered. For \texttt{SVRG}, the theoretically optimal inner loop size depends on both $L$ and $\mu$. However, $\mu$ is not always known. Moreover, even when an estimate is available,  as is the case in regularized problems with an explicit strongly convex regularizer, the estimate can often be very loose. Because of these issues, one often chooses the inner loop size in a suboptimal way, such as by setting it to $n$ or $\cO(n)$. 

{\bf Two loopless methods.} In this paper we address the above issues by developing  {\em loopless} variants of both \texttt{SVRG} and \texttt{Katyusha}; we refer to them as \texttt{L-SVRG} and \texttt{L-Katyusha}, respectively. In these methods, we dispose of the outer loop and replace its role by a {\em biased coin-flip}, to be performed in every step of the methods, used to trigger the computation of the gradient $\nabla f(w^k)$ via a pass over the data. In particular, in each step, with (a small) probability $p>0$ we perform a full pass over data and update the reference gradient $\nabla f(w^k)$. With probability $1-p$ we keep the previous reference gradient. This procedure can alternatively be interpreted as {\em having an outer loop of a random length}. However, the resulting methods are easier to write down, comprehend and analyze.

{\bf Fast rates are preserved.} We show that  \texttt{L-SVRG} and \texttt{L-Katyusha} enjoy the same fast theoretical rates as their loopy forefathers. Our proofs are different and the complexity results more insightful. 

For  \texttt{L-SVRG} with fixed stepsize $\eta = \nicefrac{1}{6L}$ and probability $p = \nicefrac{1}{n}$,
we  show (see Theorem~\ref{thm:SVRG}) that for the Lyapunov function
\begin{equation}
\textstyle \Phi^k \eqdef \norm{x^{k} - x^*}^2	+  \tfrac{4\eta^2}{pn}\sum\limits_{i=1}^n\norm{\nabla f_i(w^k) - \nabla f_i(x^*)}^2. \label{eq:Lyap1}
\end{equation}
we get
$\Exp{\Phi^k} \leq \epsilon \Phi^0$
as long as $k = \cO\left(\left(n+ \nicefrac{L}{\mu} \right) \log \nicefrac{1}{\epsilon}\right).$ In contrast, the classical \texttt{SVRG} result shows convergence of the expected functional suboptimality $\Exp{f(x^k)-f(x^*)}$ to zero  at the same rate.  Note that the classical result follows from our theorem by utilizing the inequality
$f(x^k) -f(x^*) \leq \nicefrac{L}{2}\|x^k-x^*\|^2,$
which is a simple consequence of $L$--smoothness. However, our result provides a deeper insight into the behavior of the method. In particular, it follows that the gradients $\nabla f_i(w^k)$ at the reference points $w^k$ converge to the gradients at the optimum. This is a key intuition behind the workings of \texttt{SVRG}, one not revealed by the classical analysis. Hereby we close the gap in the theoretical understanding of the the \texttt{SVRG} convergence mechanism. Moreover, our theory predicts that as long as $p$ is chosen in the (possibly very large) interval
\begin{equation} \label{eq:intervalXX} \min \left\{\nicefrac{c}{n}, \nicefrac{c\mu}{L}\right\}  \leq p \leq \max\left\{\nicefrac{c}{n}, \nicefrac{c\mu}{L}\right\},\end{equation}
where $c = \Theta(1)$, \texttt{L-SVRG} will enjoy the optimal complexity $\cO\left(\left(n+ \nicefrac{L}{\mu}\right) \log \nicefrac{1}{\epsilon}\right)$. In the ill-conditioned regime $\nicefrac{L}{\mu}\gg n$, for instance, we roughly have $p\in [\nicefrac{\mu}{L}, \nicefrac{1}{n}]$. This is in contrast with the (loopy/standard) \texttt{SVRG} method the outer loop of which needs to be of the size $\approx \nicefrac{L}{\mu}$. To the best of our knowledge, \texttt{SVRG} does not enjoy this rate for an outer loop of size $n$ (or any value independent of $\mu$, which is often not known in practice), even though this is the setting most often used in practice. Several authors have tried to establish such a result, but without success. We thus answer an  open problem since 2013, the inception of \texttt{SVRG}.

For \texttt{L-Katyusha} with stepsize $\eta=\frac{\theta_2}{(1+\theta_2)\theta_1}$ we show convergence of the Lyapunov function 
\begin{equation}
\label{def:Lyap2}
	\Psi^k = \cZ^{k}
	+
	\cY^{k}
	+
	\cW^{k},
\end{equation}
where
$
\cZ^k = \tfrac{L(1+\eta\sigma)}{2\eta}\norm{z^{k} - x^*}^2$, $ \cY^k = \tfrac{1}{\theta_1}	(f(y^k) - f(x^*))$, and $\cW^k = 	\tfrac{\theta_2(1+\theta_1)}{p\theta_1}(f(w^k) - f(x^*))$,
and where $x^k, y^k$ and $w^k$ are iterates produced by the method, with the parameters defined by $\sigma=\nicefrac{\mu}{L}$, $\theta_1 = \min \{\sqrt{\nicefrac{2\sigma n}{3}}, \nicefrac{1}{2}\}$, $\theta_2 = \nicefrac{1}{2}$, $p = \nicefrac{1}{n}$. Our main result (Theorem~\ref{thm:2}) states  that
$\Exp{\Psi^k} \leq \epsilon \Psi^0$ as long as $k=\cO((n + \sqrt{\nicefrac{nL}{\mu}} )\log \nicefrac{1}{\epsilon} ).$

{\bf Simplified analysis.} Advantage of the loopless approach is that a {\em single iteration analysis is sufficient to establish  convergence}. In contrast, one needs to perform elaborate aggregation across the inner loop to prove the convergence of the original loopy methods.

{\bf Superior empirical behaviour.} We show through extensive numerical testing on both synthetic and real data that our loopless methods are superior to their loopy variants.  We show through experiments that \texttt{L-SVRG} is {\em very robust to the choice of $p$ from the  optimal interval} \eqref{eq:intervalXX} predicted by our theory.  Moreover, {\em even the worst case for \texttt{L-SVRG} outperforms the best case for \texttt{SVRG}.} This shows how further randomization can significantly speed up and stabilize the algorithm.

{\bf Notation.} Throughout the whole paper we use conditional expectation $\E{\cX \;|\; x^k, w^k}$ for \texttt{L-SVRG} and $\E{\cX \;|\; y^k, z^k, w^k}$ for \texttt{L-Katyusha}, but for simplicity we will denote these expectations as $\E{\cX}$. If $\E{\cX}$ refers to unconditional expectation, it is directly mentioned.

\section{\SM}

In this section we describe in detail the Loopless \texttt{SVRG} method (\texttt{L-SVRG}), and its convergence.

{\bf The algorithm.} The \texttt{L-SVRG} method, formalized as Algorithm~\ref{alg:1}, is inspired by the original \texttt{SVRG} \cite{johnson2013accelerating} method. We remove the outer loop present in \texttt{SVRG}  and instead use a probabilistic update of the full gradient.\footnote{This idea was independently explored in \cite{hofmann2015variance}; we have learned about this work after a first draft of our paper was finished.} This update can be also seen in a way that outer loop size is generated by geometric distribution similar to \cite{S2GD, SCSG}.

\begin{algorithm}[h]
	\caption{\SM}
	\label{alg:1}
	\begin{algorithmic}
		\STATE {\bfseries Parameters:} stepsize $\eta>0$, probability $p\in (0,1]$ 
		\STATE {\bf Initialization:} $x^0 = w^0 \in \R^d$
		\FOR{ $k=0,1,2,\ldots$ }
		\STATE{$g^k = \nabla f_i(x^k) - \nabla f_i(w^k) + \nabla f(w^k)$} \hfill ($i \in \{1,\ldots, n\}$ is sampled uniformly at random)
		\STATE{$x^{k+1} = x^k - \eta g^k$}
		\STATE{$w^{k+1} = \begin{cases}
			x^{k}& \text{with probability } p\\
			w^k& \text{with probability } 1-p
			\end{cases}$
		}
		\ENDFOR
	\end{algorithmic}
\end{algorithm}

Note that the reference point $w^k$ (at which a full gradient is computed) is updated in each iteration with probability $p$ to the current iterate $x^k$, and is left unchanged with probability $1-p$. Alternatively, the probability $p$ can be seen as a parameter that controls the expected time before next full pass over data. To be more precise, the expected time before next full pass over data is $\nicefrac{1}{p}$.  Intuitively, we wish to keep $p$ small so that full passes over data are computed rarely enough. As we shall see next, the simple choice $p=\nicefrac{1}{n}$ leads to complexity identical to that of original \texttt{SVRG}.

{\bf Convergence theory.} A key role in the analysis is played by the {\em gradient learning} quantity
\begin{equation}
\label{def:D^k_alg_1}
\textstyle	\cD^k \eqdef \tfrac{4\eta^2}{pn}\sum\limits_{i=1}^n\norm{\nabla f_i(w^k) - \nabla f_i(x^*)}^2 
\end{equation}
and the Lyapunov function 
$
	\Phi^k \eqdef \norm{x^{k} - x^*}^2 +  \cD^k. 
$ The analysis involves four lemmas, followed by the main theorem. We wish to mention the lemmas as they highlight the way in which the argument works. All lemmas combined, together with the main theorem, can be proved on a single page, which underlines the simplicity of our approach. 

Our first lemma upper bounds the expected squared distance of $x^{k+1}$ from $x^*$ in terms of the same distance but  for $x^k$, function suboptimality, and variance of $g^k$. 

\begin{lemma}\label{lem:x^k_alg_1}
	We have
	\begin{equation}
	\label{eq:x^k_alg_1}
	\begin{split}
		&\E{\norm{x^{k+1} - x^*}^2} \leq (1-\eta\mu) \norm{x^k - x^*}^2  - 2\eta( f(x^k) - f(x^*)) + \eta^2 \E{\norm{g^k}^2}.
	\end{split}
	\end{equation}
\end{lemma}

In our next lemma, we further bound the variance of $g^k$ in terms of function suboptimality and $\cD^k$.

\begin{lemma}\label{lem:g^k_alg_1}
	We have
	\begin{equation}
	\label{eq:g^k_alg_1}
		\begin{split}
			\E{\norm{g^k}^2} \leq 4L(f(x^k) - f(x^*)) + \tfrac{p}{2\eta^2}\cD^k.
		\end{split}
	\end{equation}
\end{lemma}

Finally, we bound $\E{\cD^{k+1}}$ in terms of $\cD^k$ and function suboptimality.

\begin{lemma}\label{lem:D^k_alg_1}
	We have
	\begin{equation}
	\label{eq:D^k_alg_1}
		\begin{split}
			\E{\cD^{k+1}} \leq (1-p) \cD^{k} + 8L\eta^2(f(x^k) - f(x^*)).
		\end{split}
	\end{equation}
\end{lemma}

Putting the above three lemmas together naturally leads to the following result involving Lyapunov function \eqref{eq:Lyap1}.

\begin{lemma}\label{lem:conv_alg_1}
	Let the step size $\eta \leq \nicefrac{1}{6L}$. Then for all $k\geq 0$ the following inequality holds:
	\begin{equation}
	\label{eq:conv_alg_1}
			\E{\Phi^{k+1}} \leq 
			(1 - \eta\mu) \norm{x^k - x^*}^2
			+
			\left(1 - \nicefrac{p}{2}\right) \cD^k.
	\end{equation}
\end{lemma}

In order to obtain a recursion involving the Lyapunov function on the right-hand side  of \eqref{eq:conv_alg_1}			
\begin{theorem}\label{thm:SVRG}
	Let $\eta = \nicefrac{1}{6L}$, $p = \nicefrac{1}{n}$. Then $\E{\Phi^k} \leq \varepsilon\Phi^0$ as long as
$
		k \geq \cO\left( \left(n + \nicefrac{L}{\mu}\right) \log \nicefrac{1}{\varepsilon}\right) .
$
\end{theorem}
\begin{proof}
As the corollary of Lemma~\ref{lem:conv_alg_1} we have $ \E{\Phi^{k}} \leq 
			\max\left\{1 - \eta\mu, 1 - \nicefrac{p}{2}\right\} \Phi^{k-1}.$
Setting  $\eta = \nicefrac{1}{6L}$, $p = \nicefrac{1}{n}$ and unrolling conditional probability one obtains $ \E{\Phi^{k}} \leq 
			\max\left\{1 - \nicefrac{\mu}{6L}, 1 - \nicefrac{1}{2n}\right\}^k \Phi^{0},$
which concludes the proof.
\end{proof}

Note that the step size does not depend on the strong convexity parameter $\mu$ and yet the resulting complexity adapts to it. 

{\bf Discussion.} Examining \eqref{eq:conv_alg_1}, we can see that contraction of the Lyapunov function is $\max\{1 - \eta\mu, 1-\nicefrac{p}{2} \}$. Due to the limitation of $\eta \leq \nicefrac{1}{6L}$, the first term is at least $1 - \nicefrac{\eta}{6\mu}$, thus the complexity cannot better than $\cO\left(  \nicefrac{L}{\mu} \log \nicefrac{1}{\varepsilon}\right)$. In terms of total complexity (number of stochastic gradient calls), \texttt{L-SVRG} calls the stochastic gradient oracle in expectation $\cO(1+pn)$ times times in each iteration. Combining these two complexities together, one gets the total complexity  $\cO\left(\left(\nicefrac{1}{p} + n +  \nicefrac{L}{\mu} + \nicefrac{Lpn}{\mu}\right) \log \nicefrac{1}{\varepsilon}\right).$ Note that any choice of 
$p \in  \left[ \min \left\{\nicefrac{c}{n}, \nicefrac{c\mu}{L}\right\}, \max\left\{\nicefrac{c}{n}, \nicefrac{c\mu}{L}\right\}\right],$
where $c = \Theta(1)$, leads to the optimal total complexity  $\cO\left(\left( n +  \nicefrac{L}{\mu}\right) \log \nicefrac{1}{\varepsilon}\right)$. This fills the gap in \texttt{SVRG} theory, where the outer loop length (in our case $\nicefrac{1}{p}$ in expectation) needs to be proportional to $\nicefrac{L}{\mu}$. Moreover, analysis for \texttt{L-SVRG} is much simpler and provides more insights.

\section{\ASM}

In this section we describe in detail the Loopless \texttt{Katyusha} method (\texttt{L-Katyusha}), and its convergence properties.

{\bf The algorithm.} The \texttt{L-Katyusha} method, formalized as Algorithm~\ref{alg:2}, is inspired by the original \texttt{Katyusha} \cite{allen2017katyusha} method. We use the same technique as for Algorithm~\ref{alg:1}, where we remove the outer loop present in \texttt{Katyusha}  and instead use a probabilistic update of the full gradient. 

\begin{algorithm}[h]
	\caption{ \ASM}
	\label{alg:2}
	\begin{algorithmic}
		\STATE {\bfseries Parameters:} $\theta_1, \theta_2$, probability $p\in (0,1]$ 
		\STATE {\bf Initialization:} Choose $y^0=w^0=z^0\in \R^d$, stepsize $\eta = \frac{\theta_2}{(1+\theta_2)\theta_1}$ and set $\sigma = \nicefrac{\mu}{L}$ 
		\FOR{ $k=0,1,2,\ldots$ }
		\STATE{$x^k = \theta_1 z^k + \theta_2 w^k + (1-\theta_1 - \theta_2) y^k$}
		\STATE{$g^k = \nabla f_i(x^k) - \nabla f_i(w^k) + \nabla f(w^k)$} \hfill ($i \in \{1,\ldots, n\}$ is sampled uniformly at random)
		\STATE{$z^{k+1} = \frac{1}{1+\eta \sigma} \left(\eta\sigma x^k + z^k - \frac{\eta}{L}g^k\right)$}
		\STATE{$y^{k+1} = x^k + \theta_1(z^{k+1} - z^k)$}
		\STATE{$w^{k+1} = \begin{cases}
				y^k& \text{with probability } p\\
				w^k& \text{with probability } 1-p
			\end{cases}$
		}
		\ENDFOR
	\end{algorithmic}
\end{algorithm}

The exact analogy applies to the reference point $w^k$ (at which a full gradient is computed) as for \texttt{\SMs}. Instead of updating this point in a deterministic way every $m$ iteration, we use the probabilistic update with parameter $p$, when we update $w^{k+1}$ to the current iterate $y^k$ with this probability and is left unchanged with probability $1-p$.  As we shall see next, the same choice $p=\nicefrac{1}{n}$ as for  \texttt{\SMs} leads to complexity identical to that of original \texttt{Katyusha}.

{\bf Convergence theory.} In comparison to \texttt{\SMs}, we don't use {\em gradient mapping} as the key component of our analysis. Instead, we prove convergence of functional values in $y^k, w^k$ and point-wise convergence of $z^k$. This is summarized in the following Lyapunov function:
\begin{equation}
	\Psi^k = \cZ^{k}
	+
	\cY^{k}
	+
	\cW^{k},
\end{equation}
where
$
\cZ^k = \tfrac{L(1+\eta\sigma)}{2\eta}\norm{z^{k} - x^*}^2$, $ \cY^k = \tfrac{1}{\theta_1}	(f(y^k) - f(x^*))$, $\cW^k = 	\tfrac{\theta_2(1+\theta_1)}{p\theta_1}(f(w^k) - f(x^*))$. Note that even if $x^k$ is not in this function, its point-wise convergence is directly implied by the convergence of $\Psi^k$ due to the definition of $x^k$ in Algorithm~\ref{alg:2} and $L$-smoothness of $f$.

The analysis involves five lemmas, followed by the convergence summarized in the main theorem. The lemmas highlight important steps of our analysis. The simplicity of our approach is still preserved:  all lemmas and the main theorem can be proved on not more than two pages.

Our first lemma upper bounds the variance of the gradient estimator $g^k$, which eventually goes to zero as our algorithm progresses.

\begin{figure*}[t]
	\centering
	
	\subfloat{\includegraphics[width=0.25\linewidth]{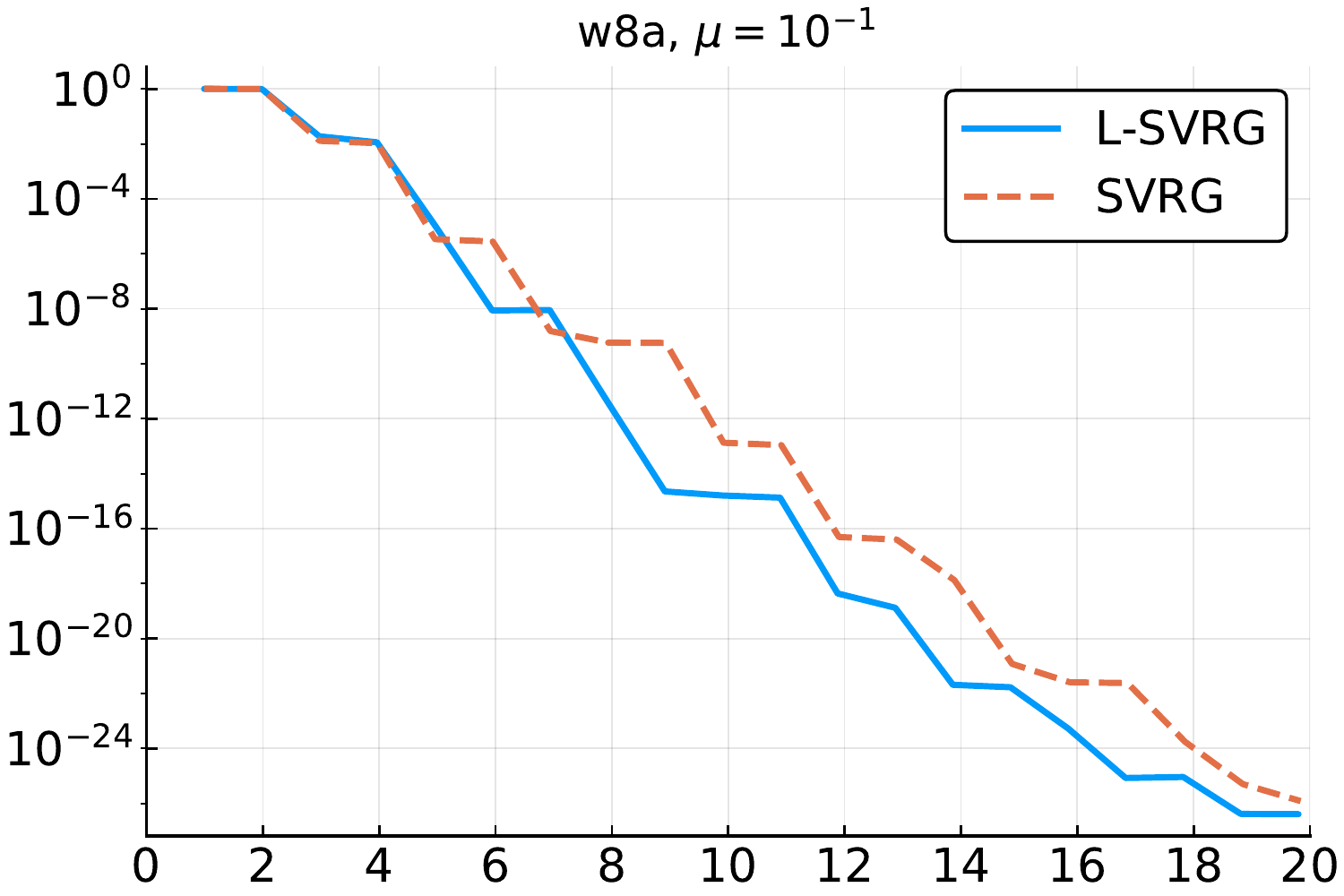}}
	\subfloat{\includegraphics[width=0.25\linewidth]{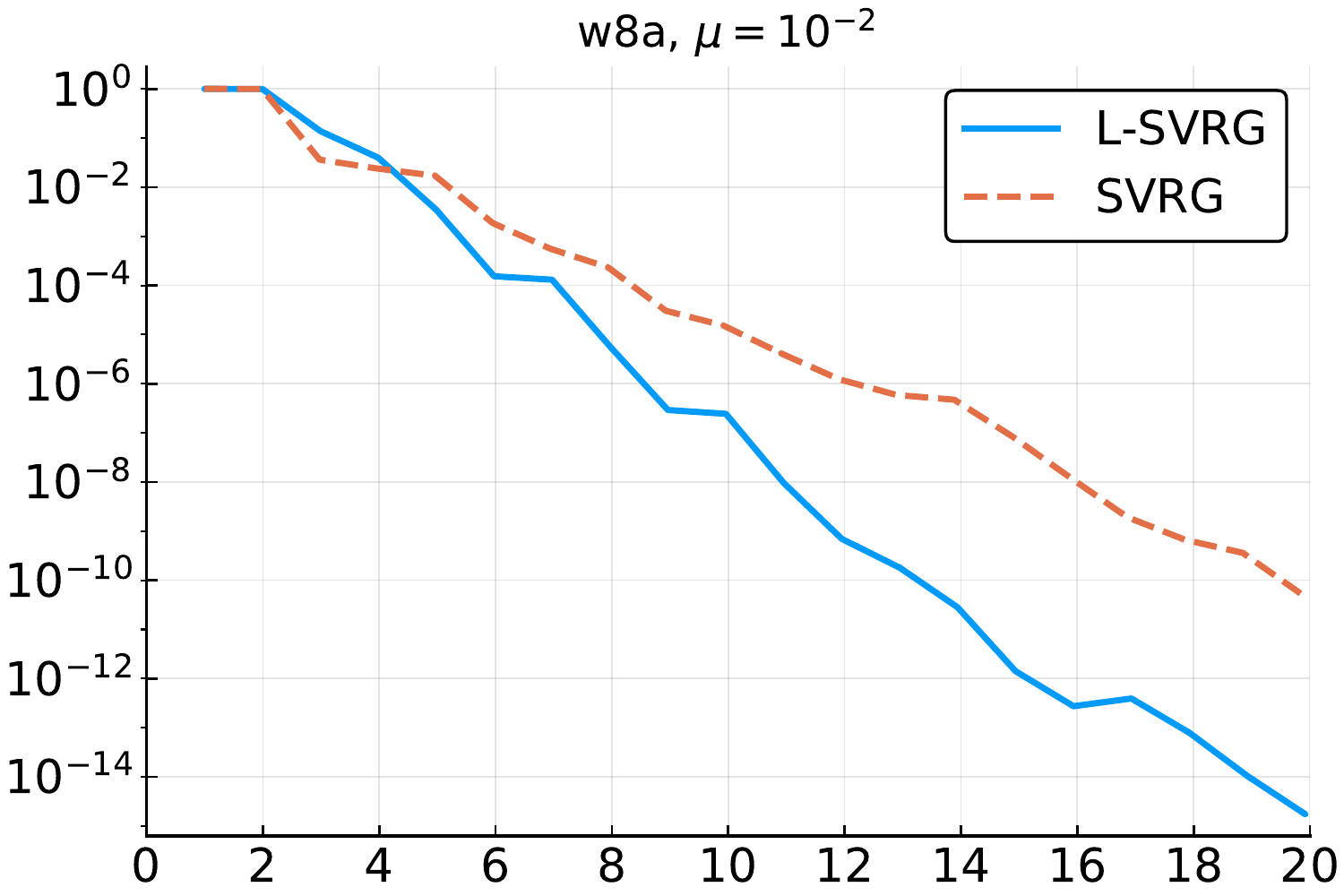}}
	\subfloat{\includegraphics[width=0.25\linewidth]{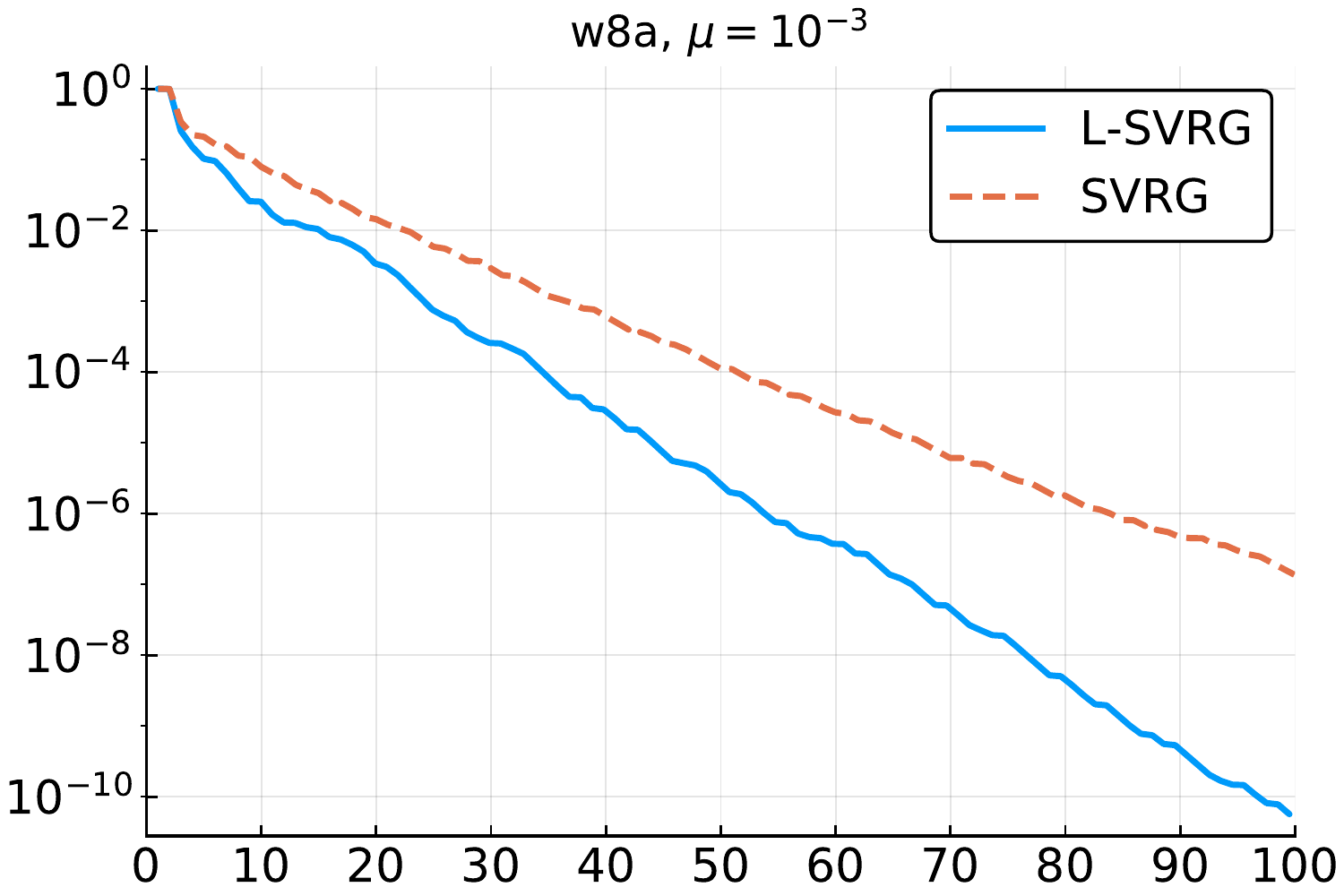}}
	\subfloat{\includegraphics[width=0.25\linewidth]{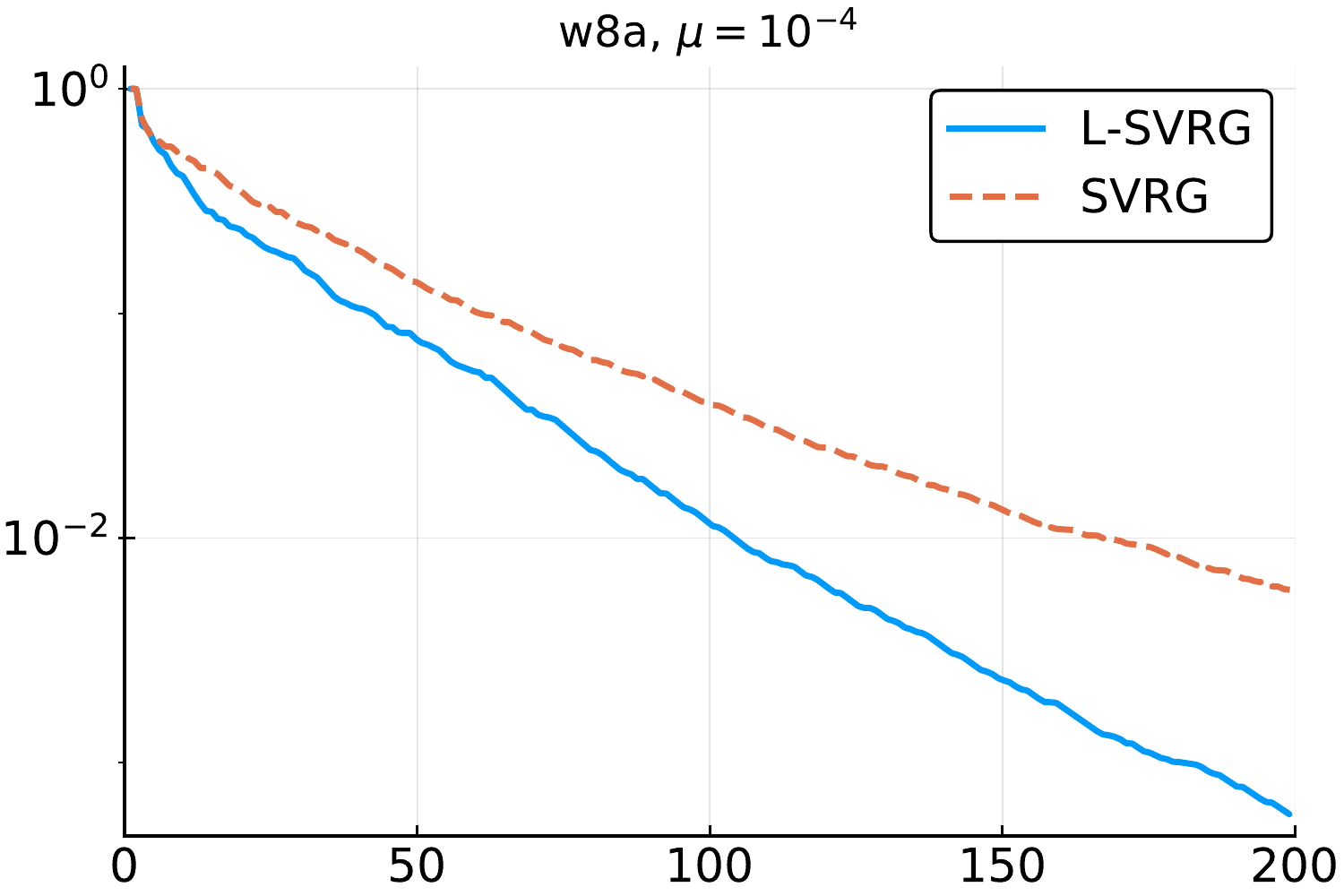}}
	
	\subfloat{\includegraphics[width=0.25\linewidth]{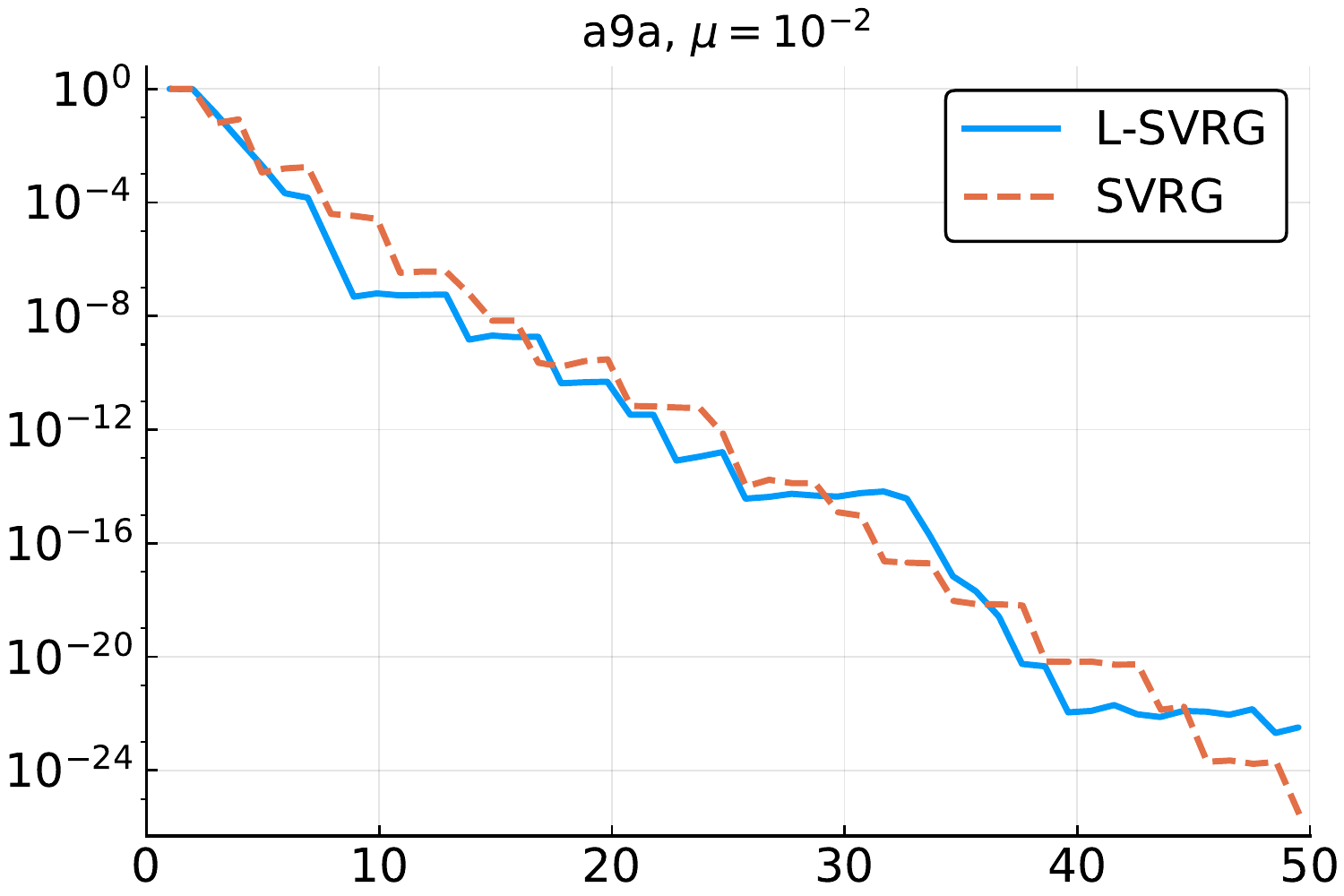}}
	\subfloat{\includegraphics[width=0.25\linewidth]{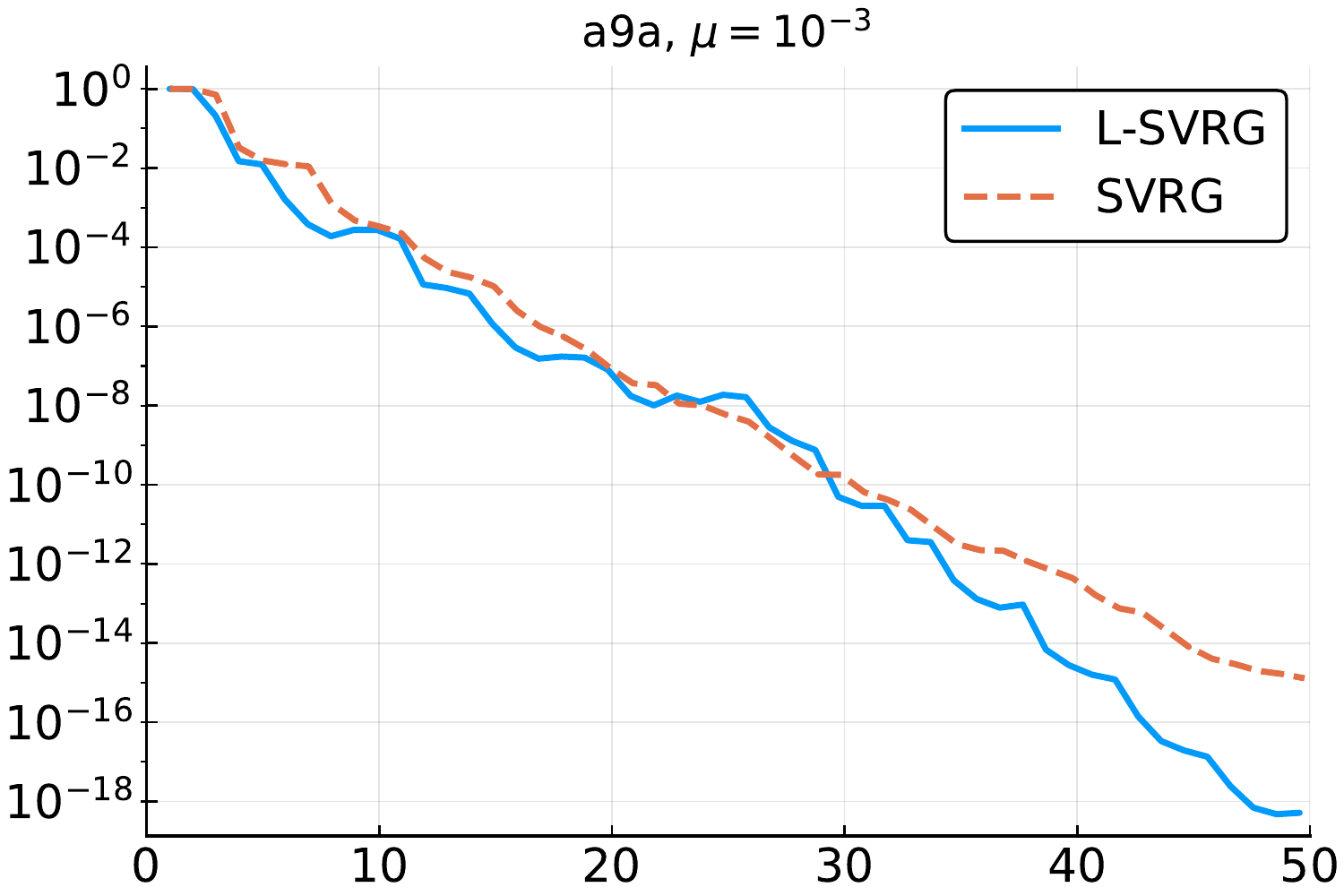}}
	\subfloat{\includegraphics[width=0.25\linewidth]{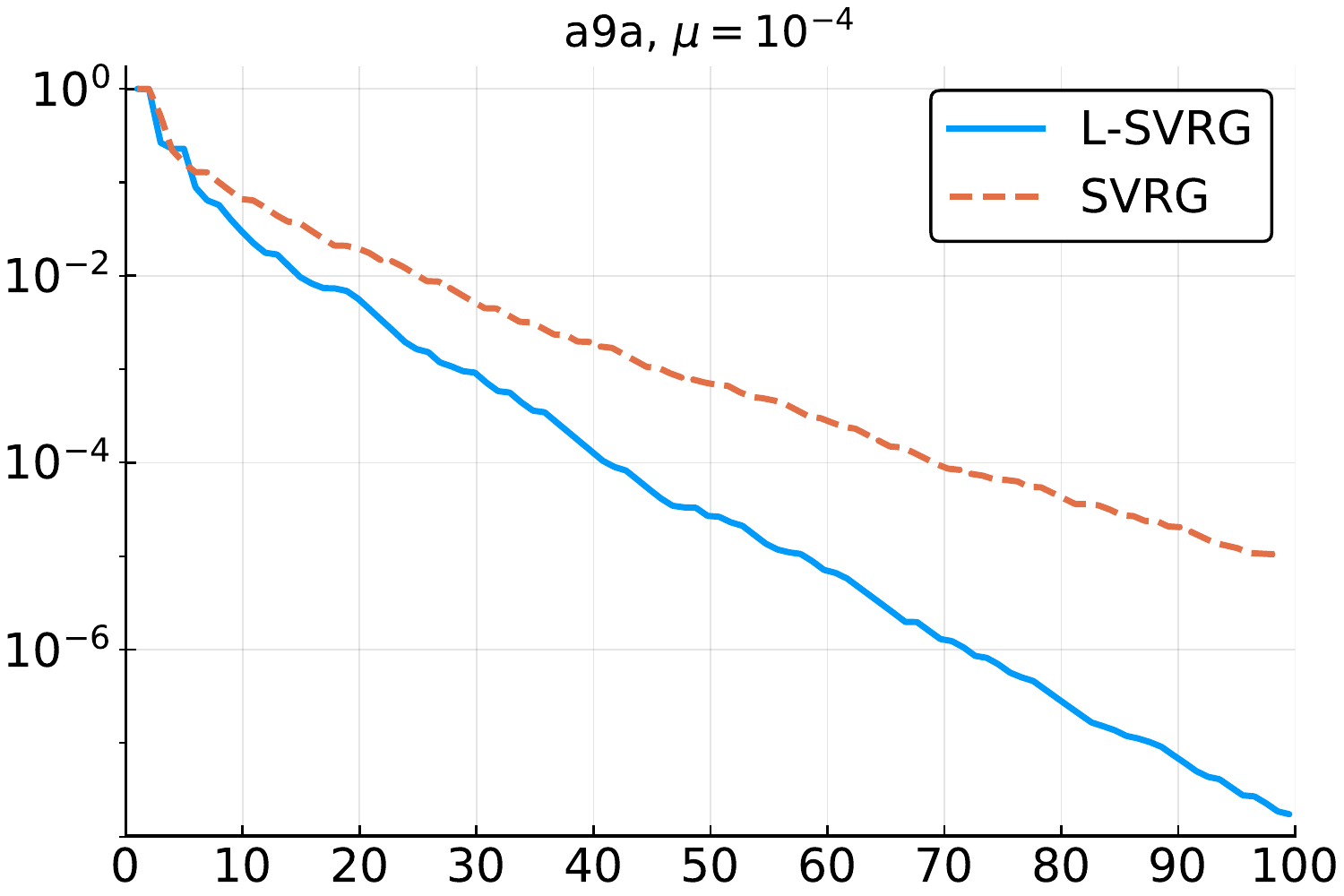}}
	\subfloat{\includegraphics[width=0.25\linewidth]{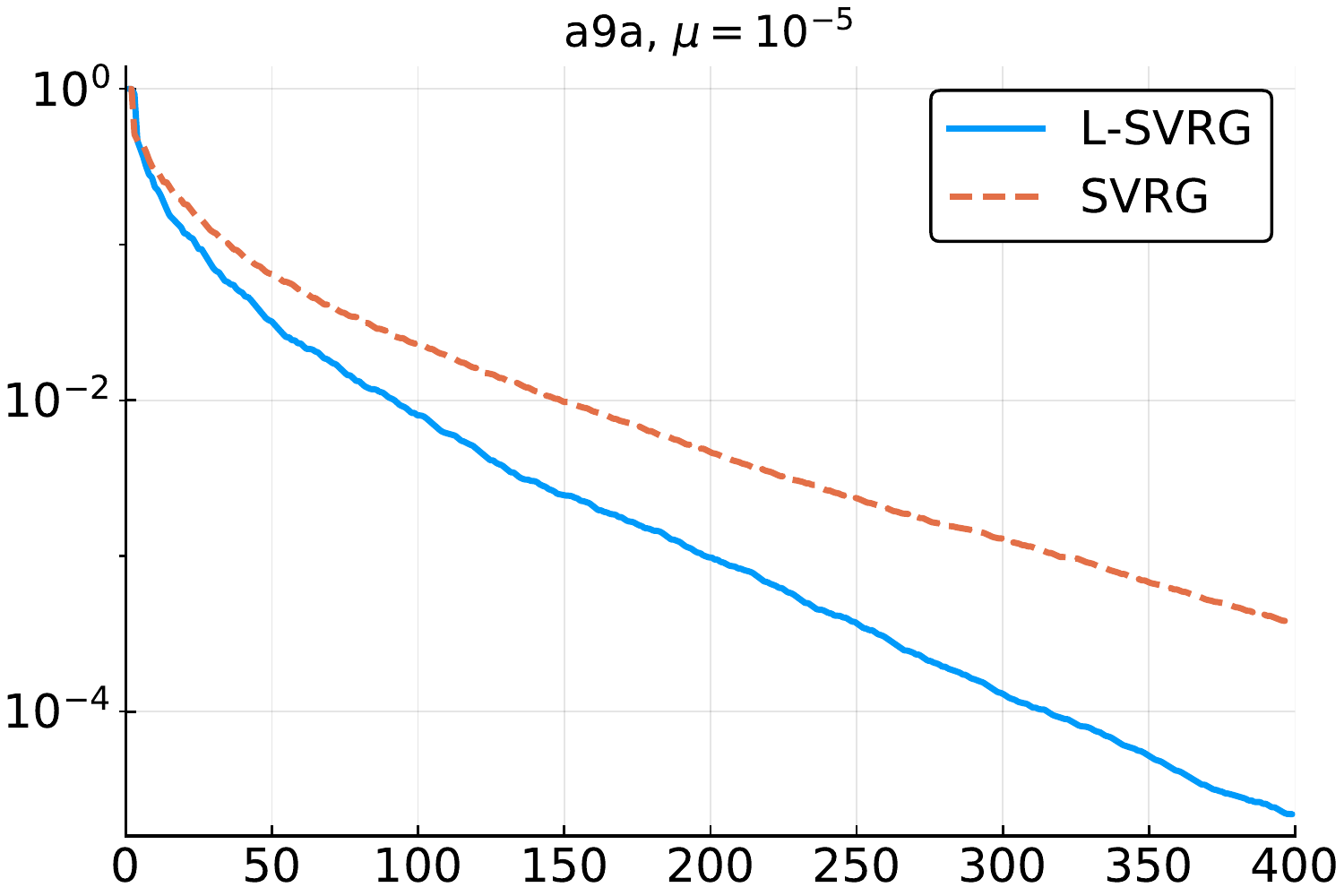}}

	\caption{Comparison of \texttt{SVRG} and  \texttt{L-SVRG} for different datasets and regularizer weights $\mu$.} 
	\label{fig:SVRG}	
\end{figure*}

\begin{lemma}\label{lem:g^k_alg_2}
	We have
	\begin{equation}
	\label{eq:g^k_alg_2}
	\begin{split}
		&\E{\norm{g^k - \nabla f(x^k)}^2} \leq
		2L \left(
			f(w^k) - f(x^k) - \dotprod{\nabla f(x^k)}{w^k - x^k}
		\right).
	\end{split}
	\end{equation}
\end{lemma}

Next two lemmas are more technical, but essential for proving the convergence. 
\begin{lemma}\label{lem:wg_alg_2}
	We have
	\begin{equation}
	\label{eq:wg_alg_2}
		\begin{split}
			&\dotprod{g^k}{x^*-z^{k+1}}
			+
			\tfrac{\mu}{2}\norm{x^k - x^*}^2 \geq
			\tfrac{L}{2\eta}\norm{z^k - z^{k+1}}^2
			+
			\cZ^{k+1}
			-
			\tfrac{1}{1+\eta\sigma}\cZ^k.
		\end{split}
	\end{equation}
\end{lemma}

\begin{lemma}\label{lem:wtheta_alg_2}
	We have
	\begin{equation}
	\label{eq:wtheta_alg_2}
		\begin{split}
			&\tfrac{1}{\theta_1}\left(f(y^{k+1}) - f(x^k)\right)
			-
			\tfrac{\theta_2}{2L\theta_1} \norm{g^k - \nabla f(x^k)}^2 \leq
			\tfrac{L}{2\eta} \norm{z^{k+1} - z^k}^2
			+
			\dotprod{g^k}{z^{k+1} - z^k}.
		\end{split}
	\end{equation}
\end{lemma}

Finally, we use the update of Algorithm~\ref{alg:2} to decompose $\cW^{k+1}$ in terms of $\cW^k$ and $\cY^k$, which is one of the main components that allow for simpler analysis than the one of original \texttt{Katyusha}.
\begin{lemma}\label{lem:W^k_alg_2}
	We have
	\begin{equation}
	\label{eq:W^k_alg_2}
	\begin{split}
	\E{\cW^{k+1}} = (1-p)\cW^k + \theta_2(1+\theta_1) \cY^k.
	\end{split}
	\end{equation}
\end{lemma}

Putting all lemmas together, we obtain the following contraction of the Lyapunov function  \eqref{def:Lyap2}.

\begin{lemma}\label{lem:conv_alg_2}
	Let $\theta_1, \theta_2 > 0$, $\theta_1 + \theta_2 \leq 1$, $\sigma = \frac{\mu}{L}$ and $\eta = \frac{\theta_2}{(1+\theta_2)\theta_1}$, then we have 
	\begin{equation}
	\label{eq:conv_alg_2}
	\begin{split}
	&\E{
		\cZ^{k+1}
		+
		\cY^{k+1}
		+
		\cW^{k+1}
	}\leq
	\tfrac{1}{1+\eta\sigma}\cZ^k
	+
	(1-\theta_1(1-\theta_2))\cY^k  +
	\left(1 - \tfrac{p\theta_1}{1+\theta_1}\right)\cW^k.
	\end{split}
	\end{equation}
\end{lemma}
In order to obtain a recursion involving the Lyapunov function on the right-hand side  of \eqref{eq:conv_alg_2}		
\begin{theorem}\label{thm:2}
	Let $\theta_1 = \min \{\sqrt{\nicefrac{2\sigma n}{3}}, \nicefrac{1}{2} \}$, $\theta_2 = \nicefrac{1}{2}$, $p = \nicefrac{1}{n}$.
	Then $\E{\Psi^{k}} \leq   \varepsilon\Psi^0$ after the following number of iterations:
$
		k = \cO( (n + \sqrt{\nicefrac{nL}{\mu}}) \log \nicefrac{1}{\varepsilon}) .
$
\end{theorem}

\begin{proof}
From Lemma~\ref{lem:conv_alg_2} we get 
$
\E{\Psi^{k+1}} \leq \max\left\{\nicefrac{1}{(1+\eta\sigma)}, 1-\theta_1(1-\theta_2),1 - \nicefrac{p\theta_1}{(1+\theta_1)}\right\} \Psi^{k}.
$
Setting  $p = \nicefrac{1}{n}$, $\theta_1 = \min\{\sqrt{\nicefrac{2\sigma n}{3}}, \nicefrac{1}{2}\}$, $\theta_2 = \nicefrac{1}{2}$, and unrolling conditional probability one obtains 
$\E{\Psi^{k+1}} \leq (1-\theta) \E{\Psi^{k}}$, where $\theta = \min\left\{   \nicefrac{\sigma}{6 \theta_1},  \nicefrac{\theta_1}{2n}\right\}
 .$
Choosing $\sigma = \nicefrac{\mu}{L}$ concludes the proof.
\end{proof}

{\bf Discussion.} One can show by analyzing  \eqref{eq:conv_alg_2} that for ill-conditioned problems ($n < \nicefrac{L}{ \mu}$), the  iteration complexity is $\cO(\sqrt{\nicefrac{L}{\mu p}} \log \nicefrac{1}{\varepsilon} )$.
Algorithm~\ref{alg:2} calls stochastic gradient oracle $\cO(1 + pn)$ times per iteration in expectation.
Thus, the total complexity is
$\cO ( (1+pn)\sqrt{\nicefrac{L}{\mu p}} \log \nicefrac{1}{\varepsilon} )$.
One can see that $p = \Theta\left(\nicefrac{1}{n}\right)$ leads to optimal rate.

\section{Numerical Experiments}

\begin{figure*}[!t]	
	\centering

	\subfloat{\includegraphics[width=0.25\linewidth]{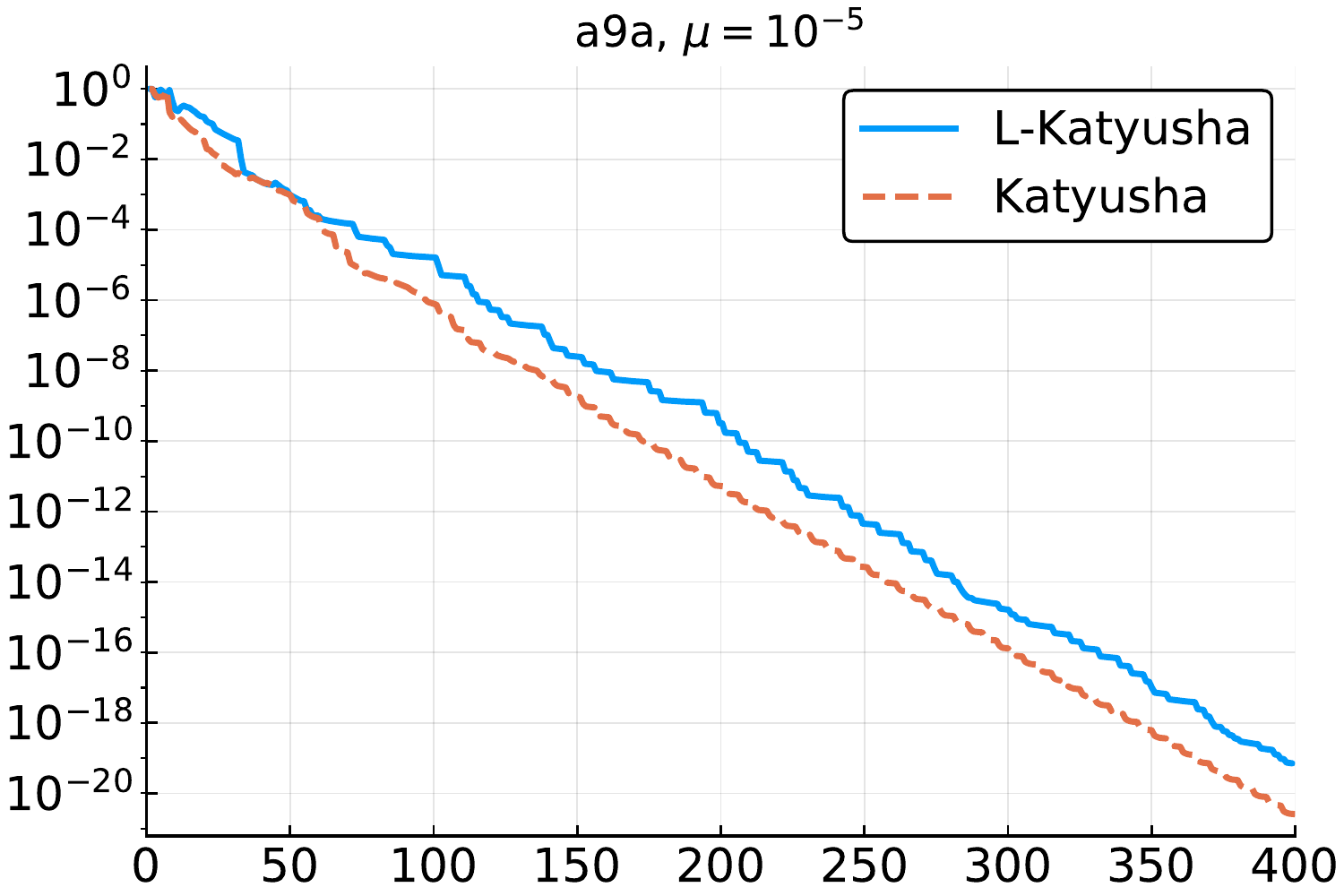}}
	\subfloat{\includegraphics[width=0.25\linewidth]{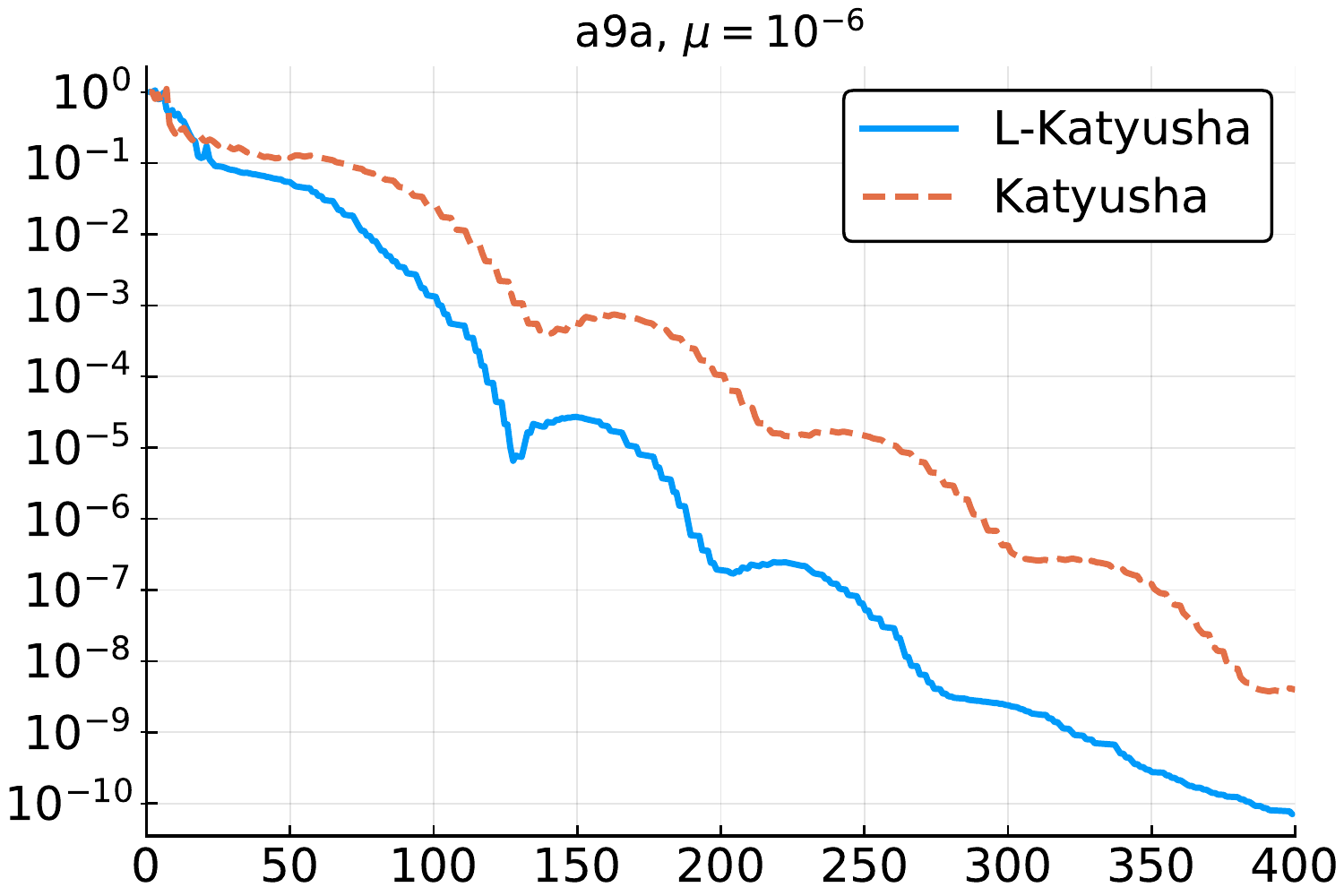}}
	\subfloat{\includegraphics[width=0.25\linewidth]{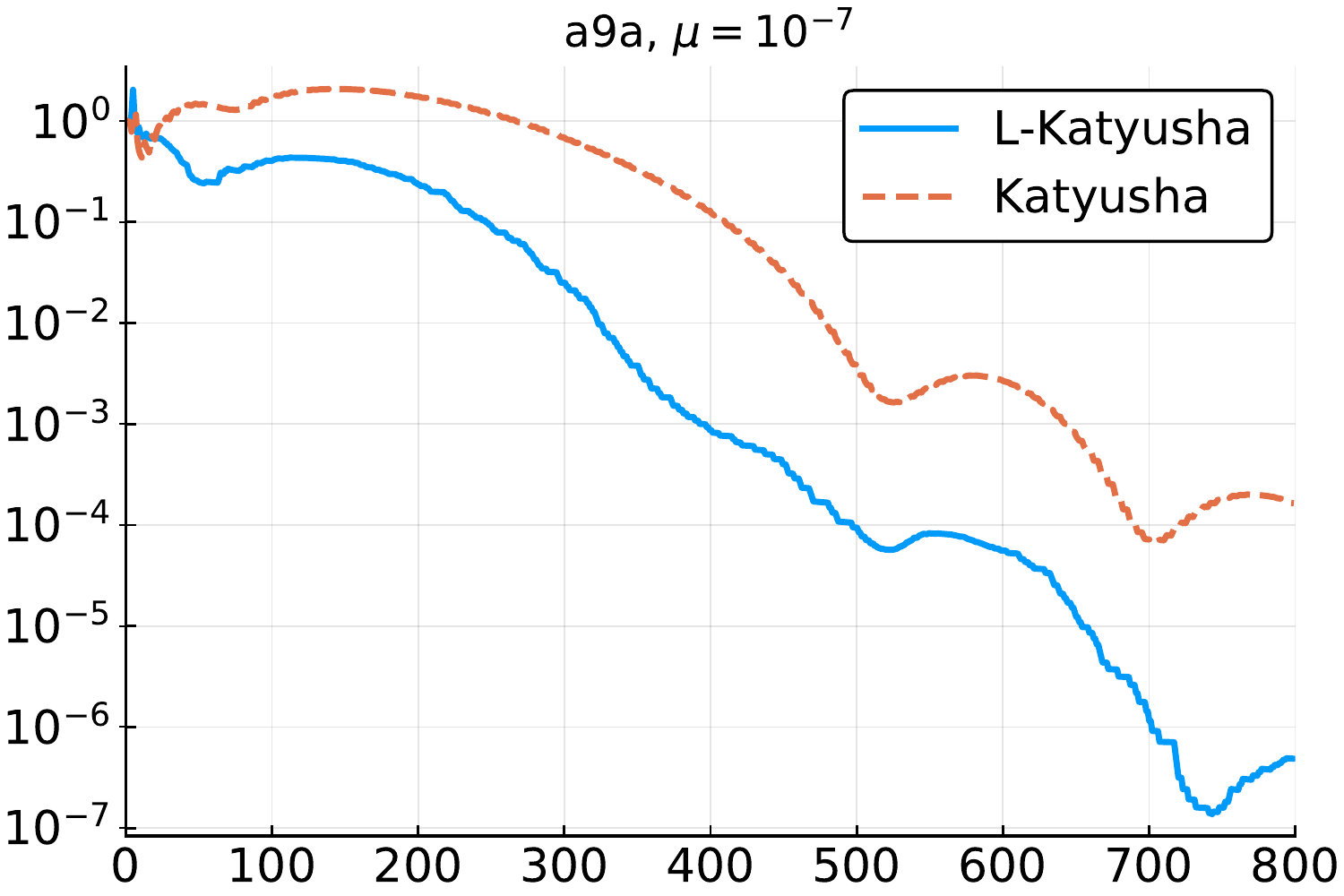}}
	
	\subfloat{\includegraphics[width=0.25\linewidth]{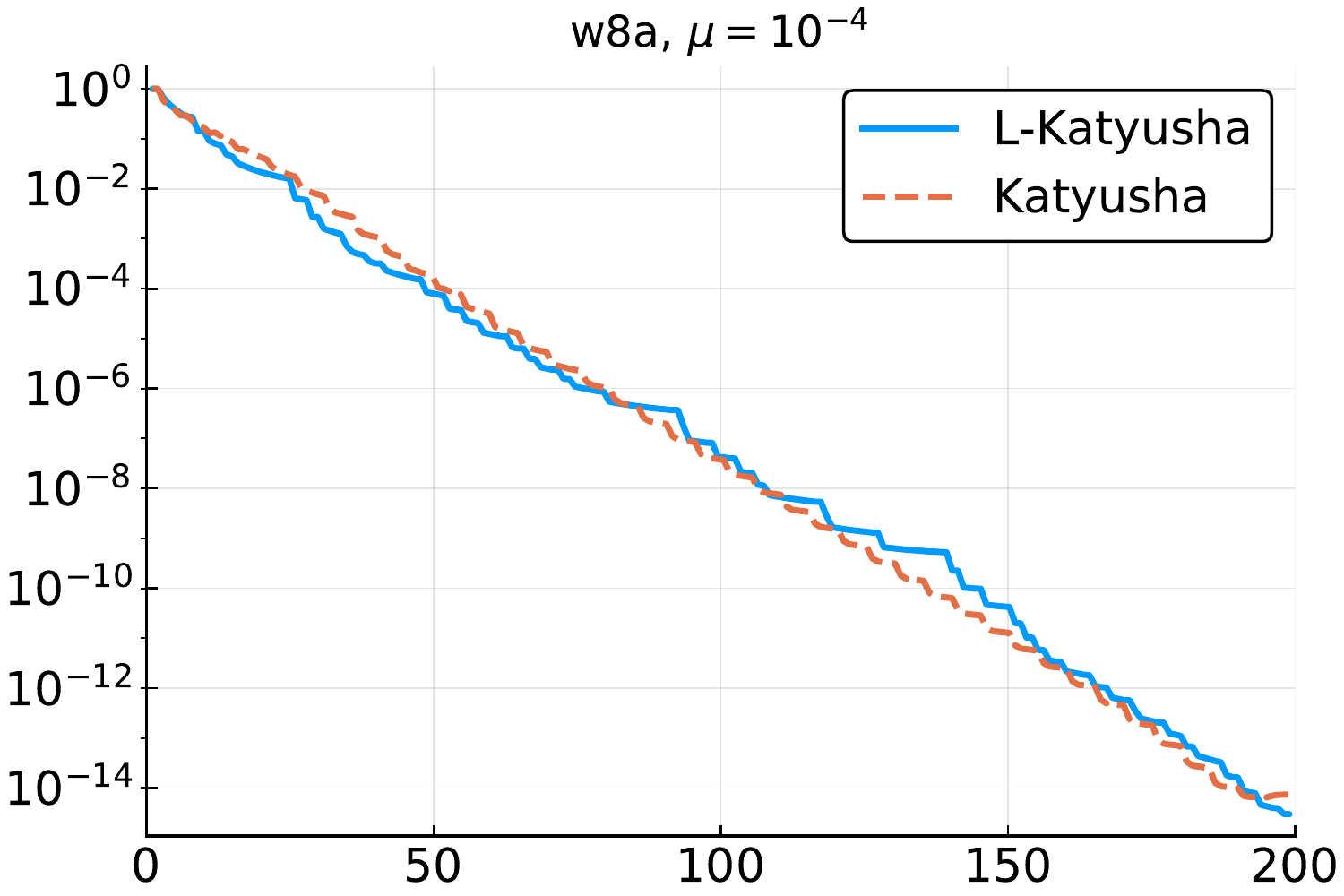}}
	\subfloat{\includegraphics[width=0.25\linewidth]{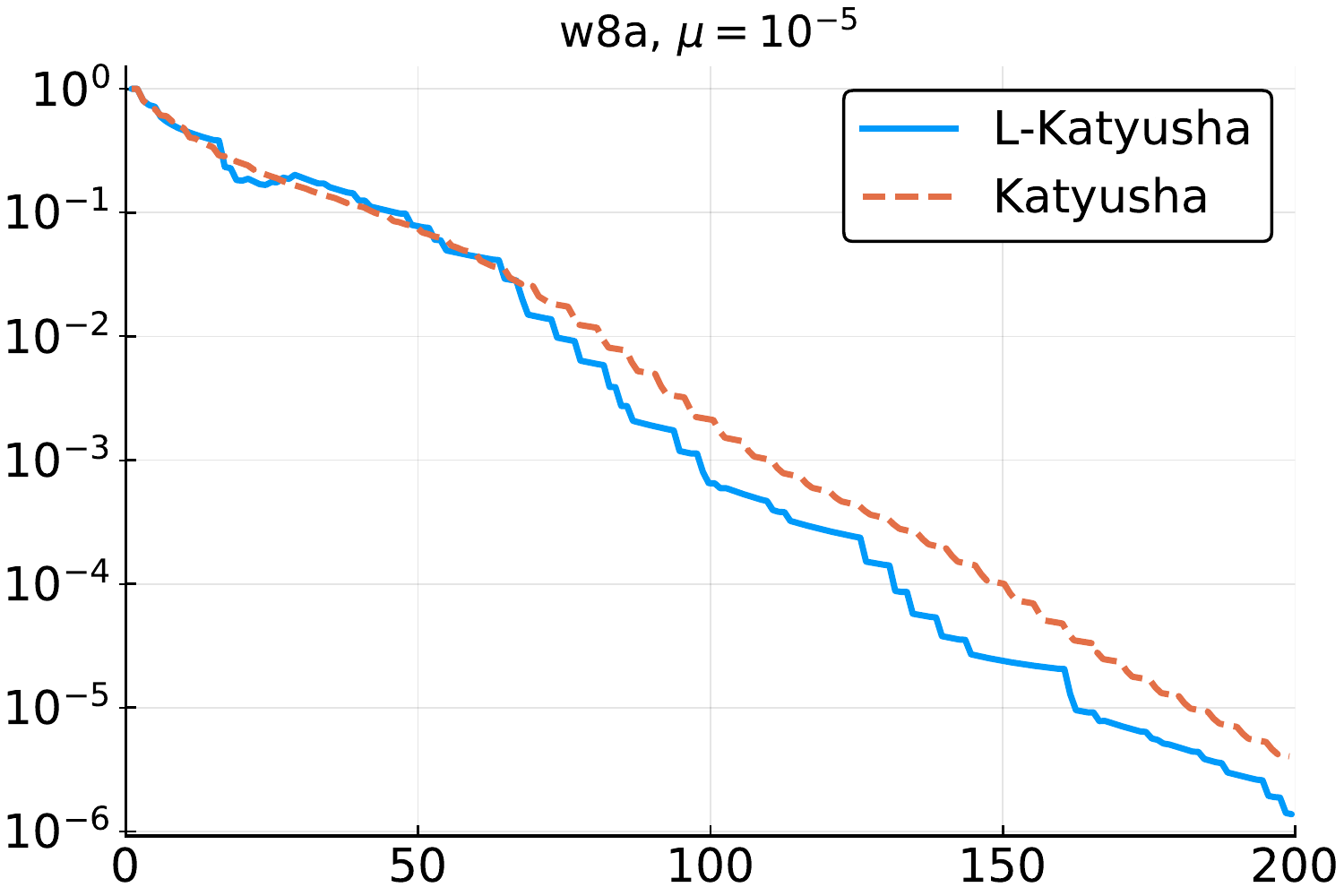}}
	\subfloat{\includegraphics[width=0.25\linewidth]{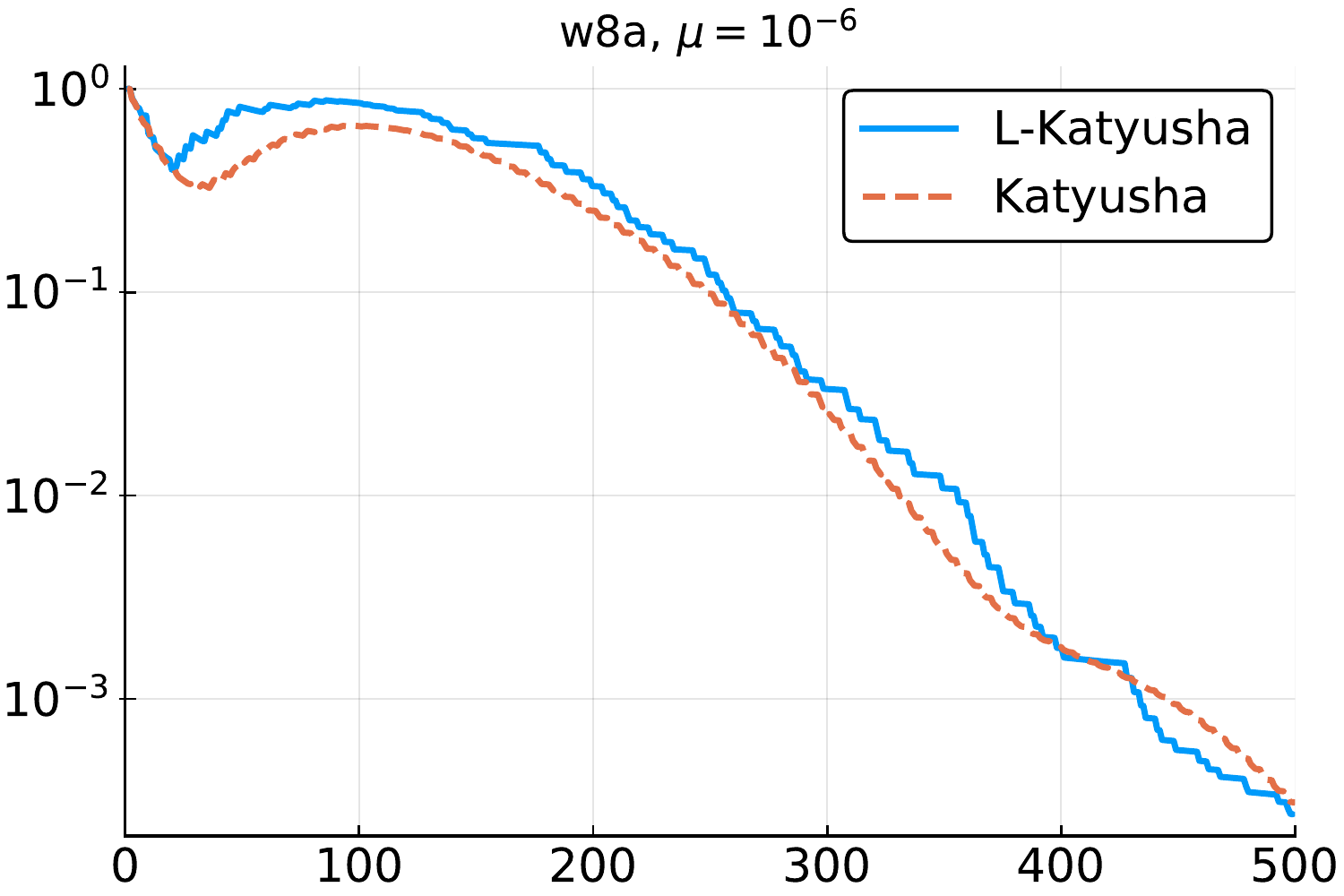}}
	
	\caption{Comparison of \texttt{Katyusha} \&  \texttt{L-Katyusha} for different datasets and regularizer weights $\mu$.} 
	\label{fig:katyusha}
\end{figure*}

\begin{figure*}[t]
	\centering

	\subfloat{\includegraphics[width=0.25\linewidth]{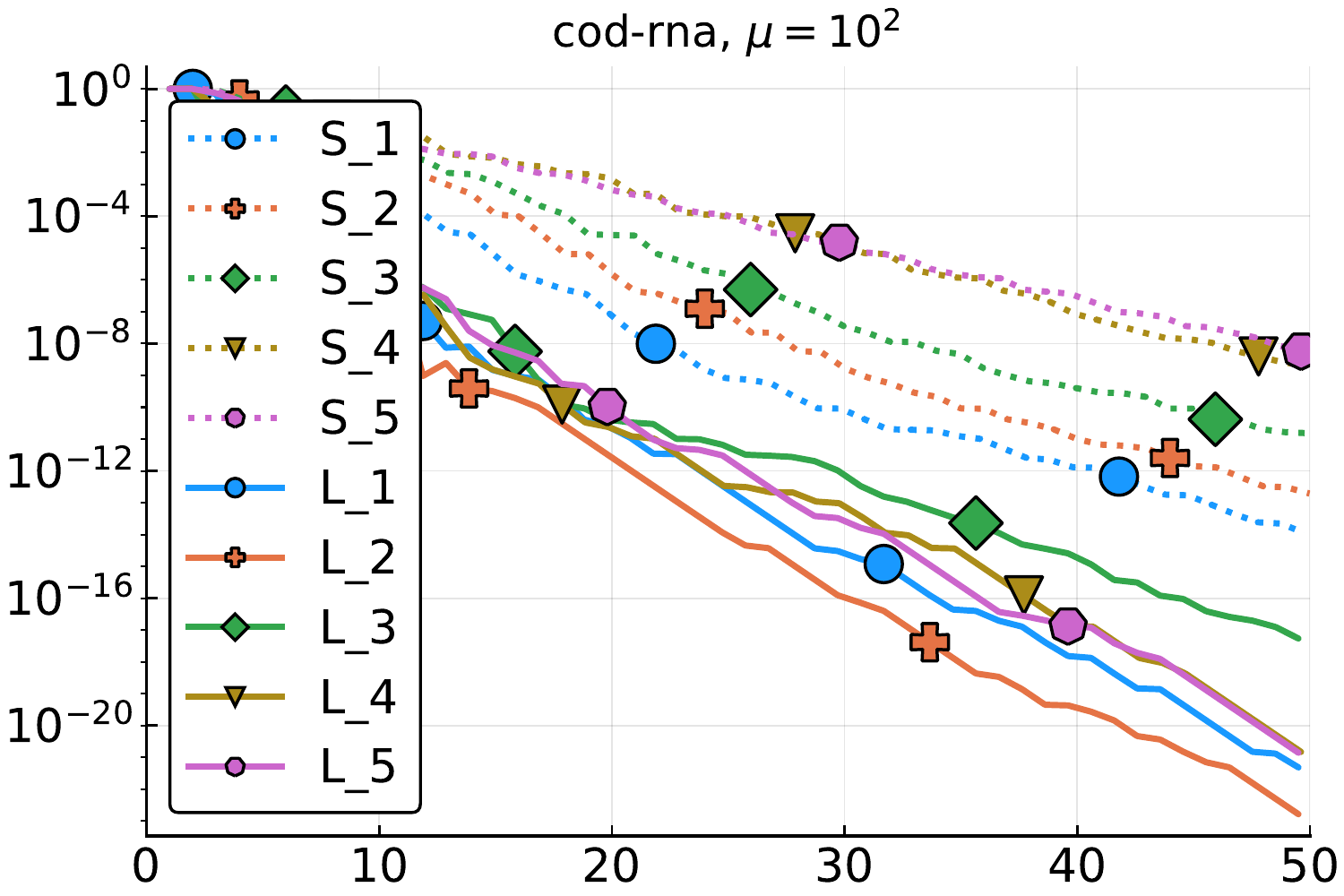}}
	\subfloat{\includegraphics[width=0.25\linewidth]{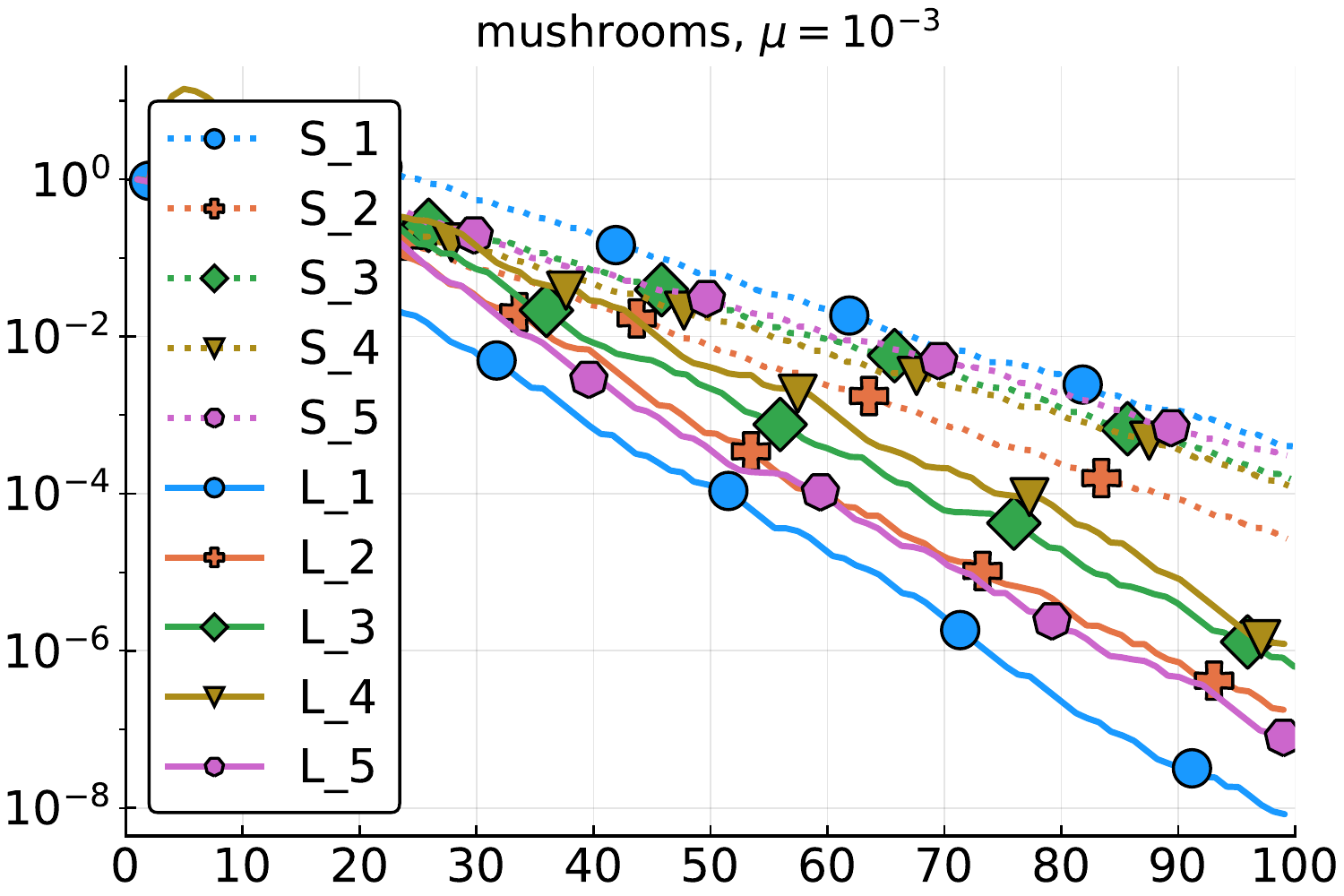}}
	\subfloat{\includegraphics[width=0.25\linewidth]{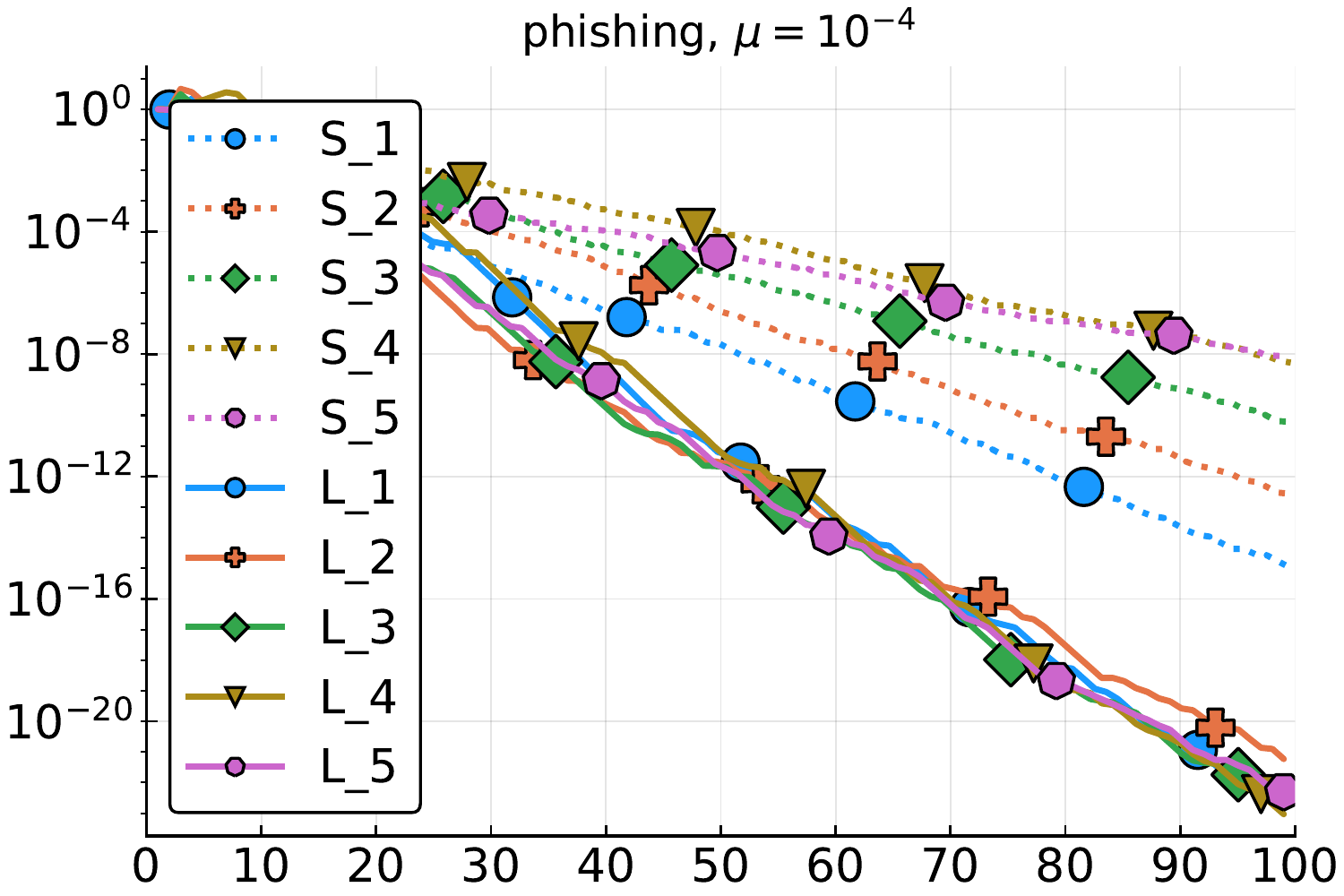}}
		
	\caption{Comparison of \texttt{SVRG}  (S) and  \texttt{L-SVRG} (L) for several choices of expected outer loop length (\texttt{L-SVRG})  or deterministic outer loop length (\texttt{SVRG}). Numbers $1$--$5$ correspondent to loop-lengths $n$,$\sqrt[4]{\kappa n^3}$, $\sqrt{\kappa n}$, $\sqrt[4]{\kappa^3n}$, $\kappa$, respectively, where $\kappa = \nicefrac{L}{\mu}$.}
\label{fig:diff_p}
\end{figure*}

In this section, we perform   experiments with logistic regression for binary classification with $L_2$ regularizer, where our loss function has the form
$
f_i(x) = \log (1 + \exp(- b_i a_i^\top x) ) + \tfrac{\mu}{2} \norm{x}^2,
$
where $a_i \in \R^d$, $b_i \in \{-1,+1\}$, $i \in [n]$. Hence, $f$  is smooth and $\mu$-strongly convex.  We use four  LIBSVM library\footnote{The LIBSVM dataset collection is available at \url{https://www.csie.ntu.edu.tw/~cjlin/libsvmtools/datasets/}}: \textit{a9a, w8a, mushrooms, phishing, cod-rna}.

We compare our methods \texttt{L-SVRG} and  \texttt{L-Katyusha} with their original version. It is well-known that whenever practical, \texttt{SAGA} is a bit faster than \texttt{SVRG}. While a comparison to \texttt{SAGA} seems natural as it also does not have a double loop structure, we position our loopless methods for applications where the high memory requirements of  \texttt{SAGA} prevent it to be applied. Thus, we do not compare to \texttt{SAGA}.

 Plots are constructed in such a way that the $y$-axis displays $\norm{x^k-x^\star}^2$ for \texttt{\SMs} and $\norm{y^k-x^\star}^2$ for \texttt{\ASMs}, where $x^\star$ were obtained by running gradient descent for a large number of epochs. The $x$-axis displays the number of epochs (full gradient evaluations). That is, $n$ computations of $\nabla f_i(x)$ equals one epoch.

{\bf Superior practical behaviour of the loopless approach.}
\label{sec:exp_superior}
Here we show that  \texttt{L-SVRG} and \texttt{L-Katyusha} perform better in  experiments than their loopy variants. In terms of theoretical iteration complexity,  both the loopy and the loopless methods are the same. However, as we can see from Figure~\ref{fig:SVRG}, the improvement of the loopless approach  can be significant. One can see that for these datasets, \texttt{L-SVRG}  {\em is always better than \texttt{SVRG}, and can be faster by several orders of magnitude!} Looking at Figure~\ref{fig:katyusha}, we see that  {\em the performance of \texttt{L-Katyusha} is at least as good as that of \texttt{Katyusha}, and  can be significantly faster in some cases.} All  parameters  of the methods were chosen as suggested by theory. For \texttt{L-SVRG} and \texttt{L-Katyusha} they are chosen  based on Theorems~\ref{thm:SVRG} and \ref{thm:2}, respectively. For \texttt{SVRG} and \texttt{Katyusha} we also choose the parameters based on the theory, as described in the original papers. The initial point $x^0$ is chosen to be the origin.

\begin{figure*}[t!]
	\centering

	\subfloat{\includegraphics[width=0.25\linewidth]{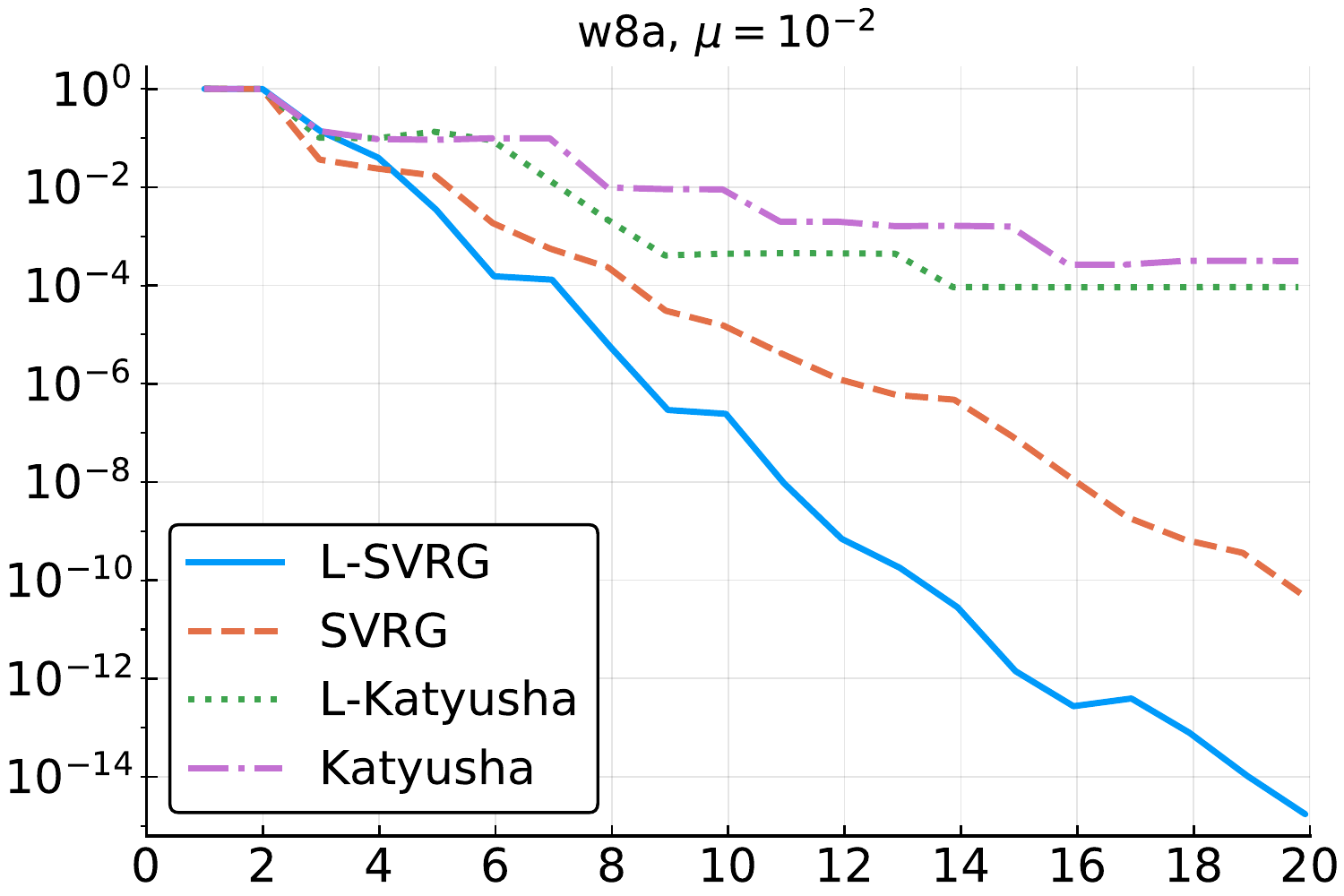}}
	\subfloat{\includegraphics[width=0.25\linewidth]{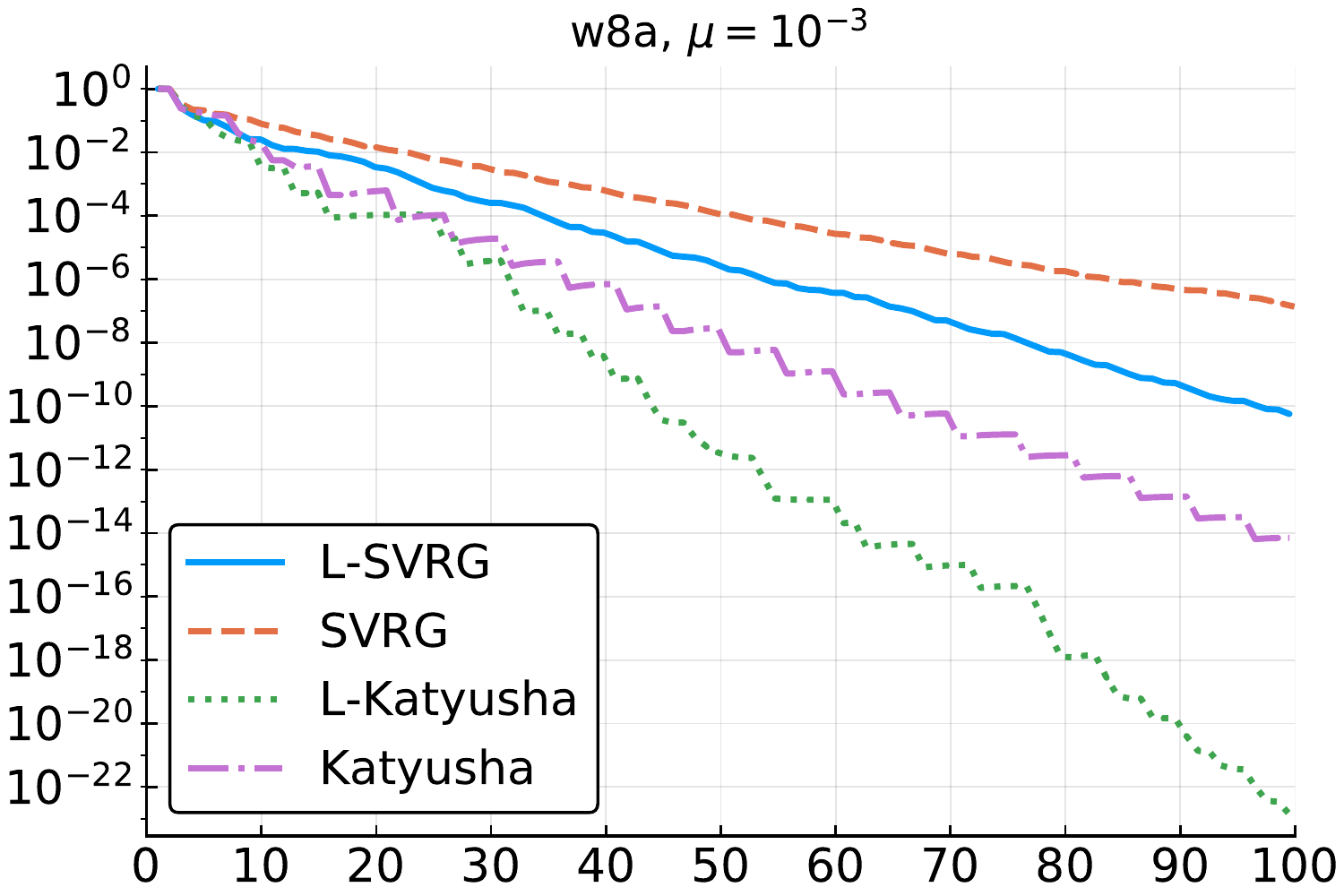}}
	\subfloat{\includegraphics[width=0.25\linewidth]{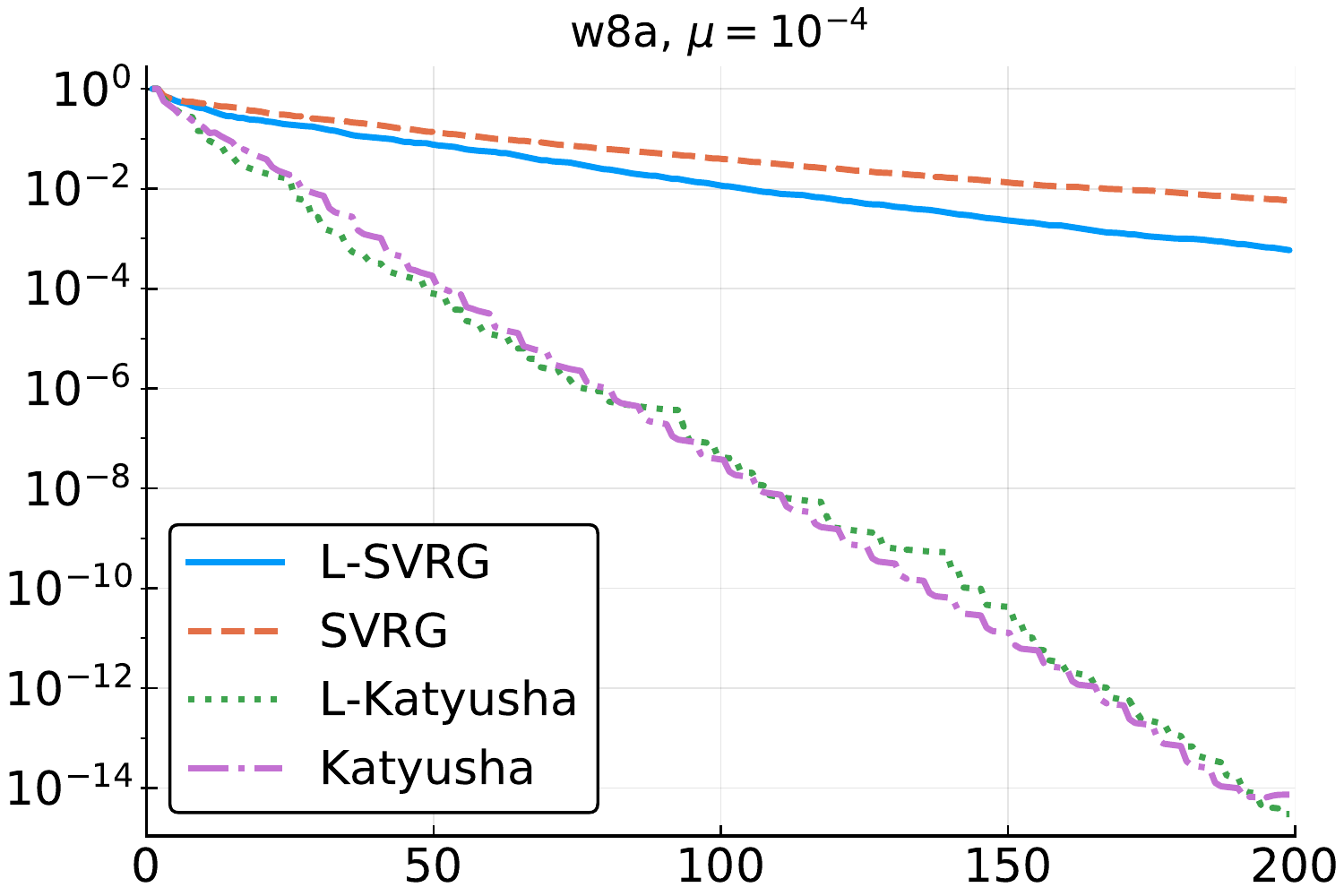}}
	\subfloat{\includegraphics[width=0.25\linewidth]{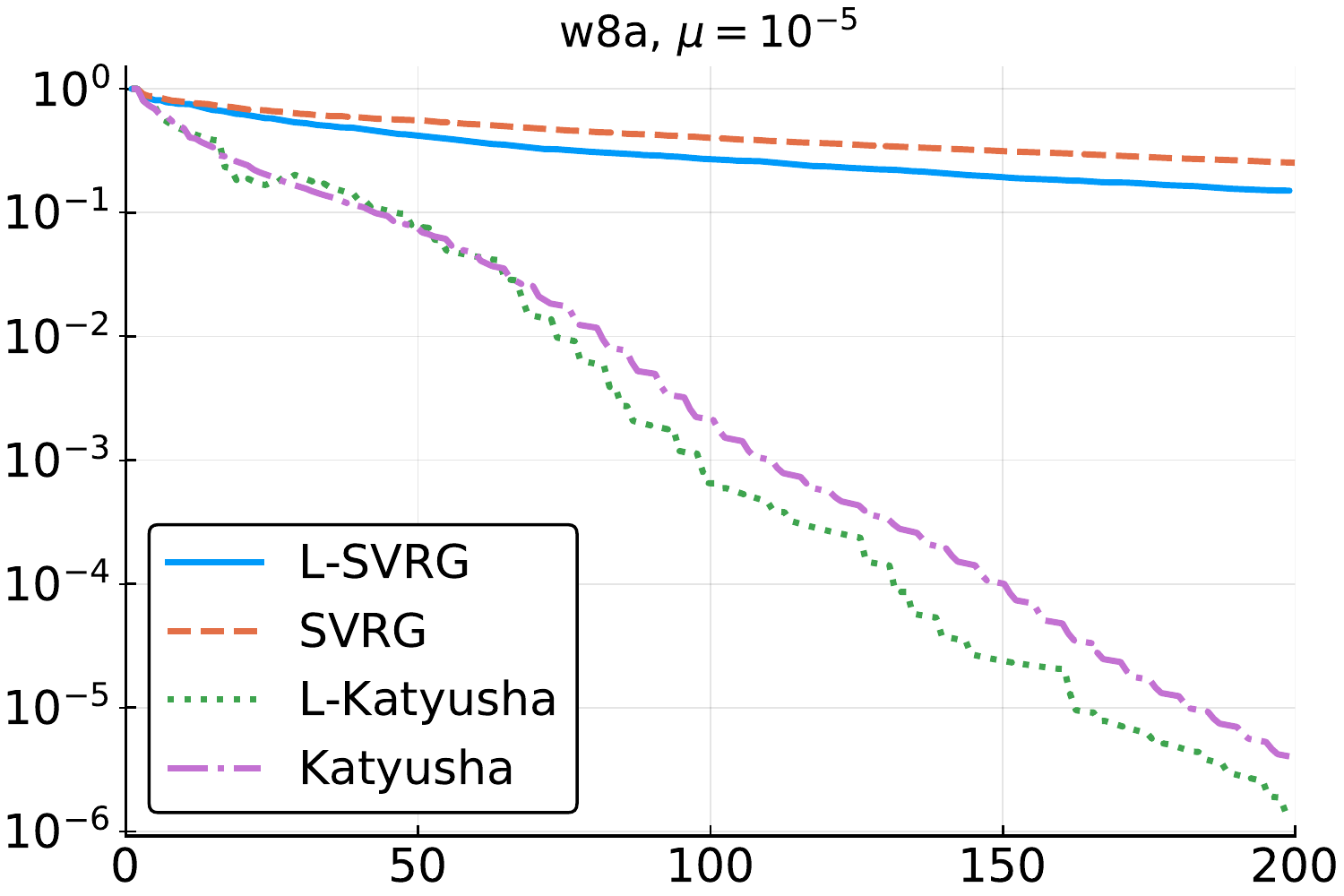}}

	\subfloat{\includegraphics[width=0.25\linewidth]{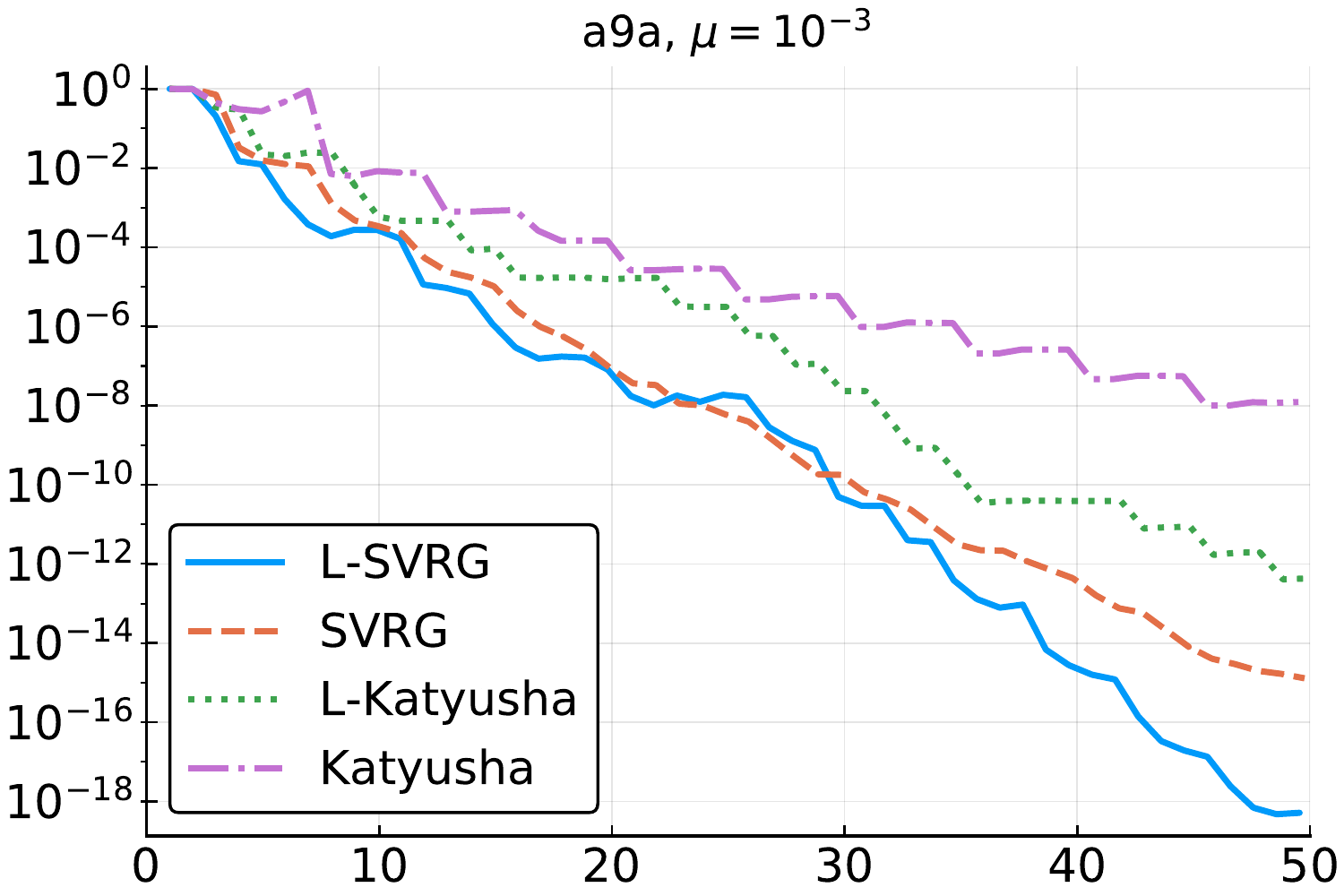}}
	\subfloat{\includegraphics[width=0.25\linewidth]{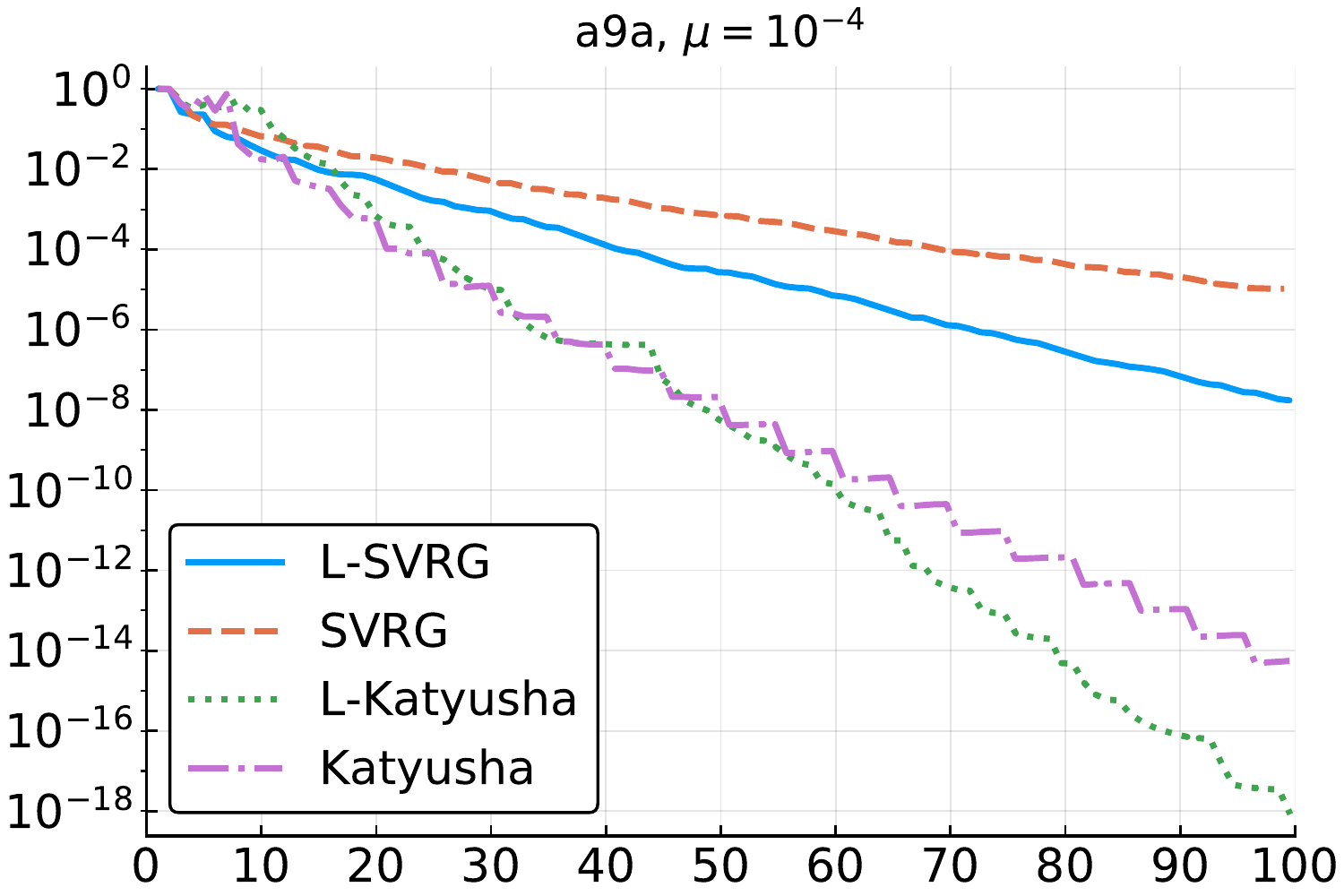}}
	\subfloat{\includegraphics[width=0.25\linewidth]{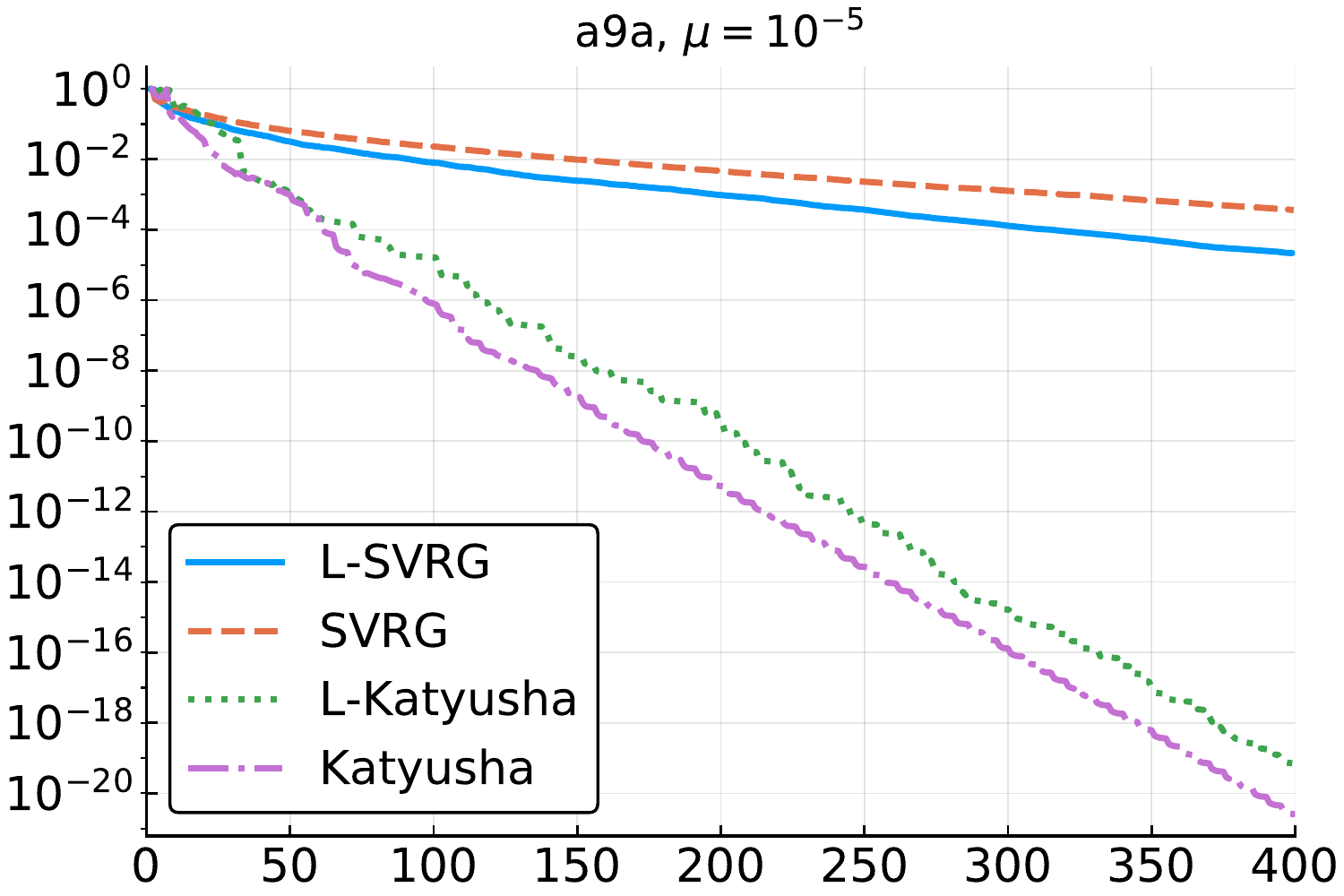}}
	\subfloat{\includegraphics[width=0.25\linewidth]{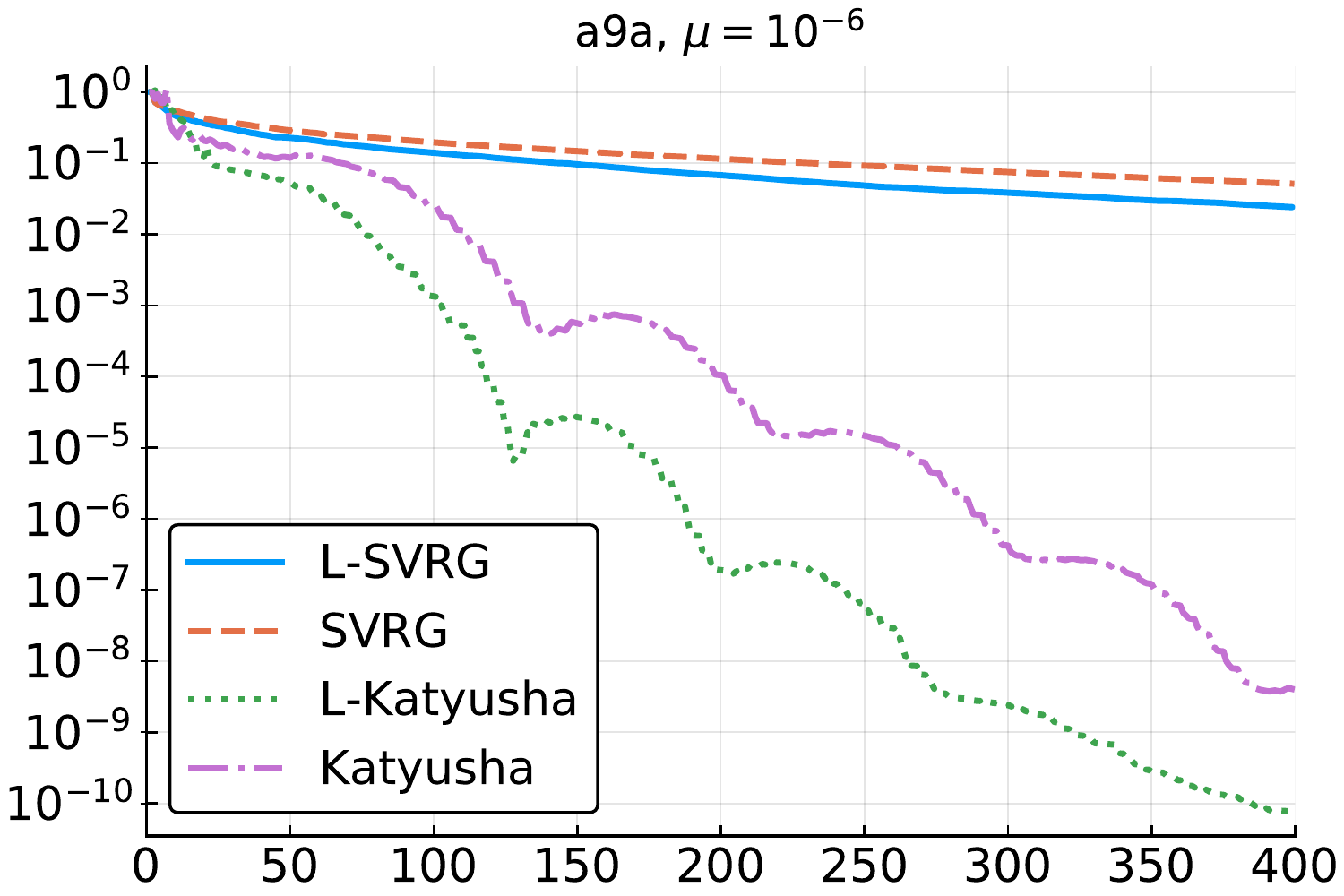}}
	
	\subfloat{\includegraphics[width=0.25\linewidth]{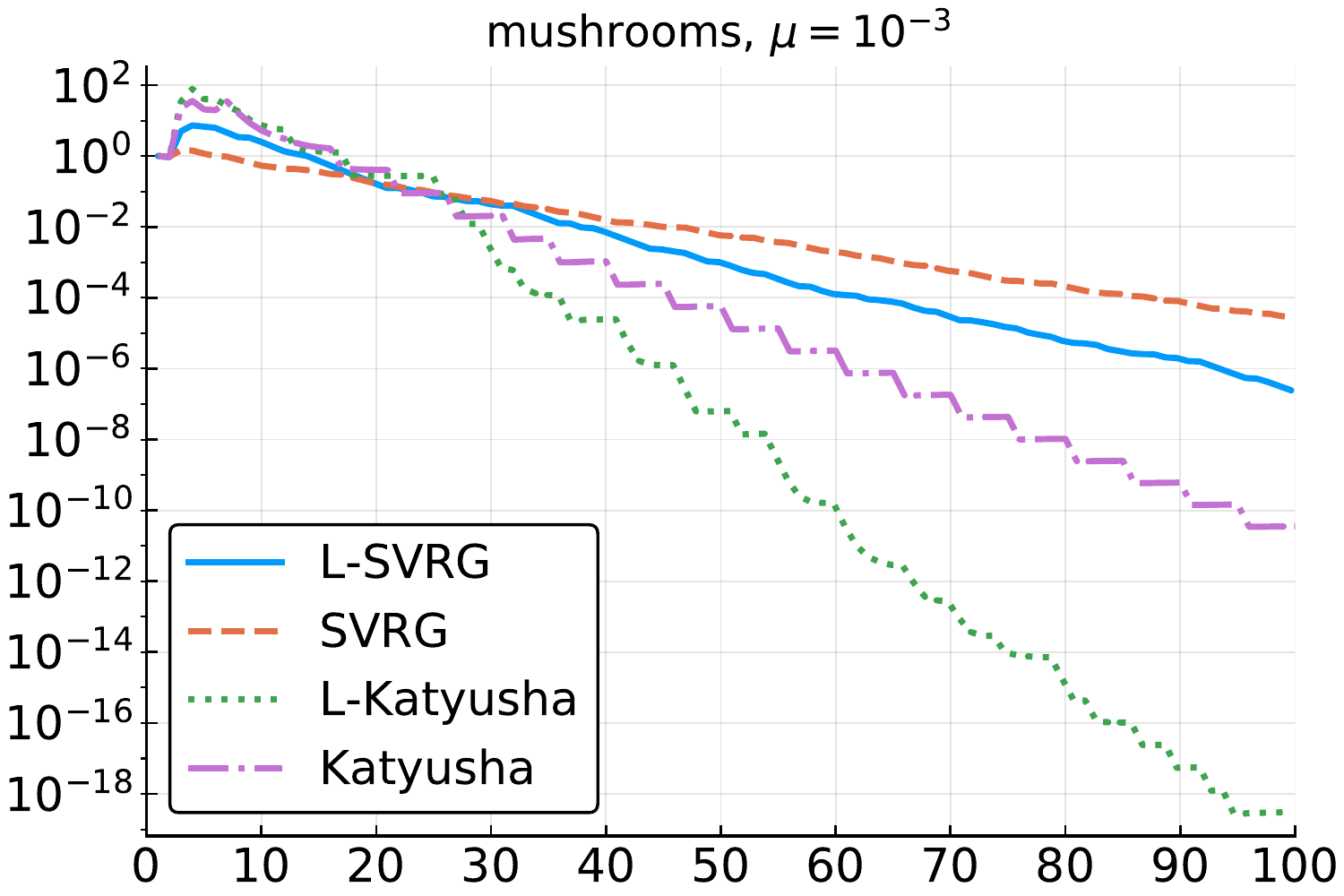}}
	\subfloat{\includegraphics[width=0.25\linewidth]{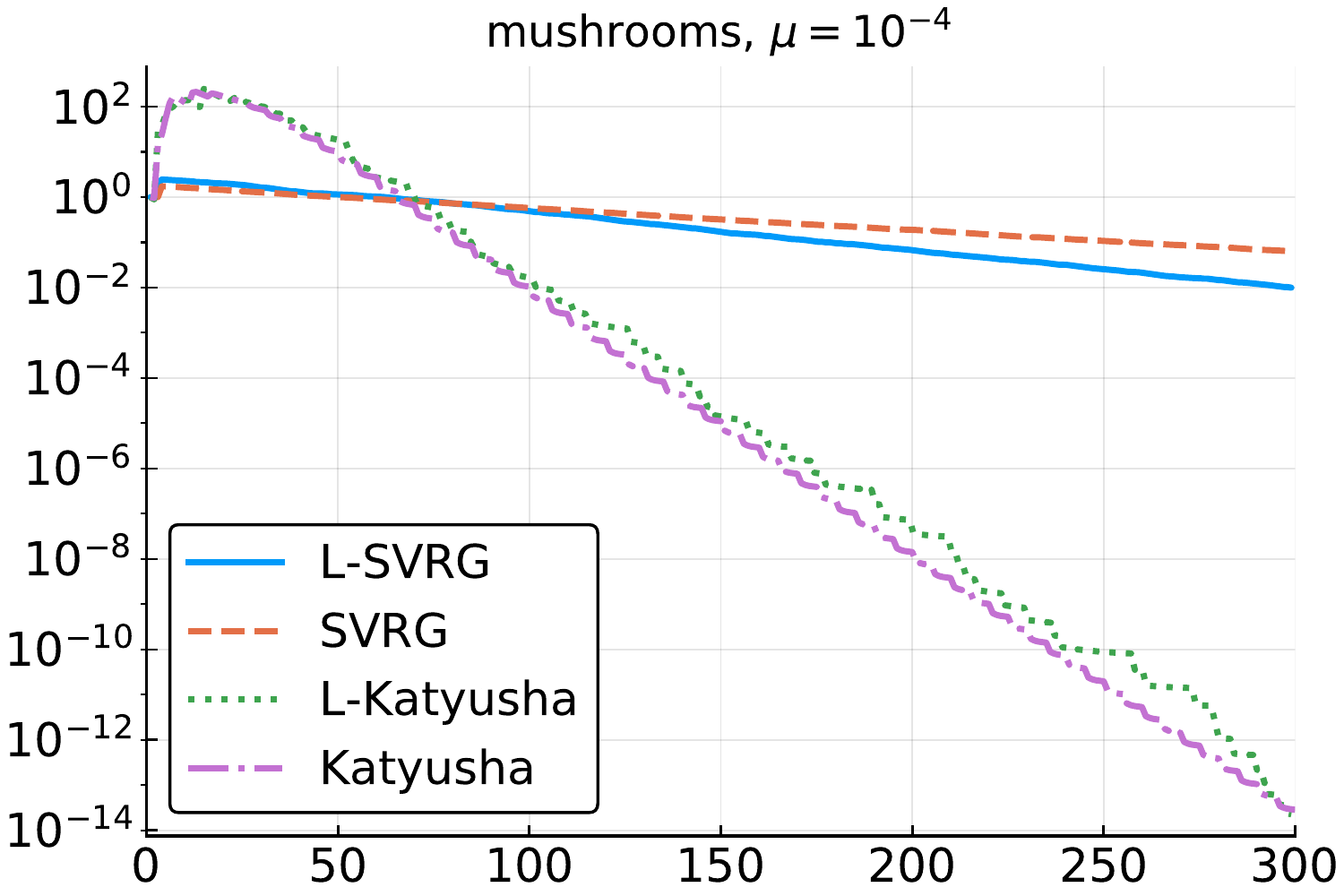}}
	\subfloat{\includegraphics[width=0.25\linewidth]{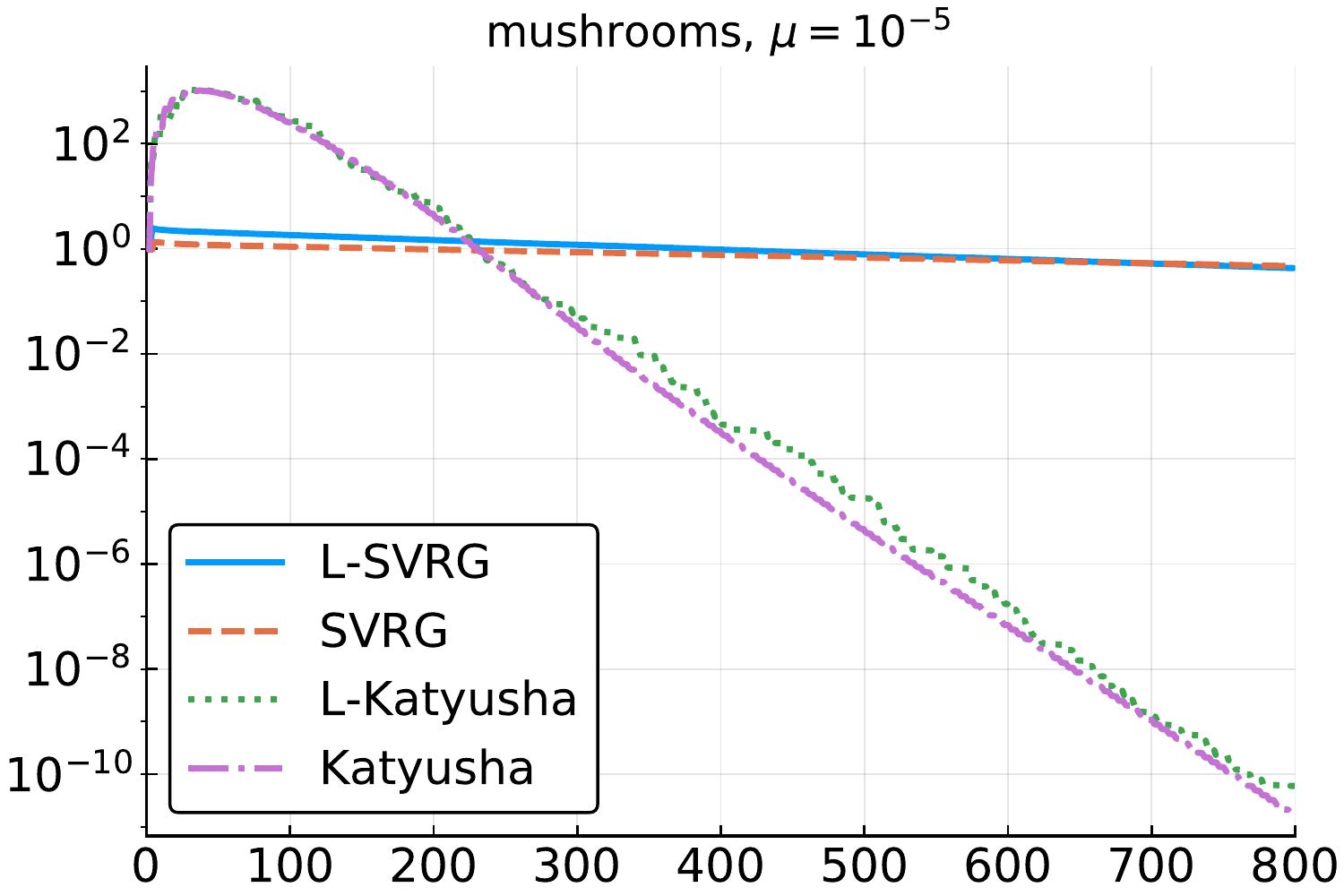}}
	\subfloat{\includegraphics[width=0.25\linewidth]{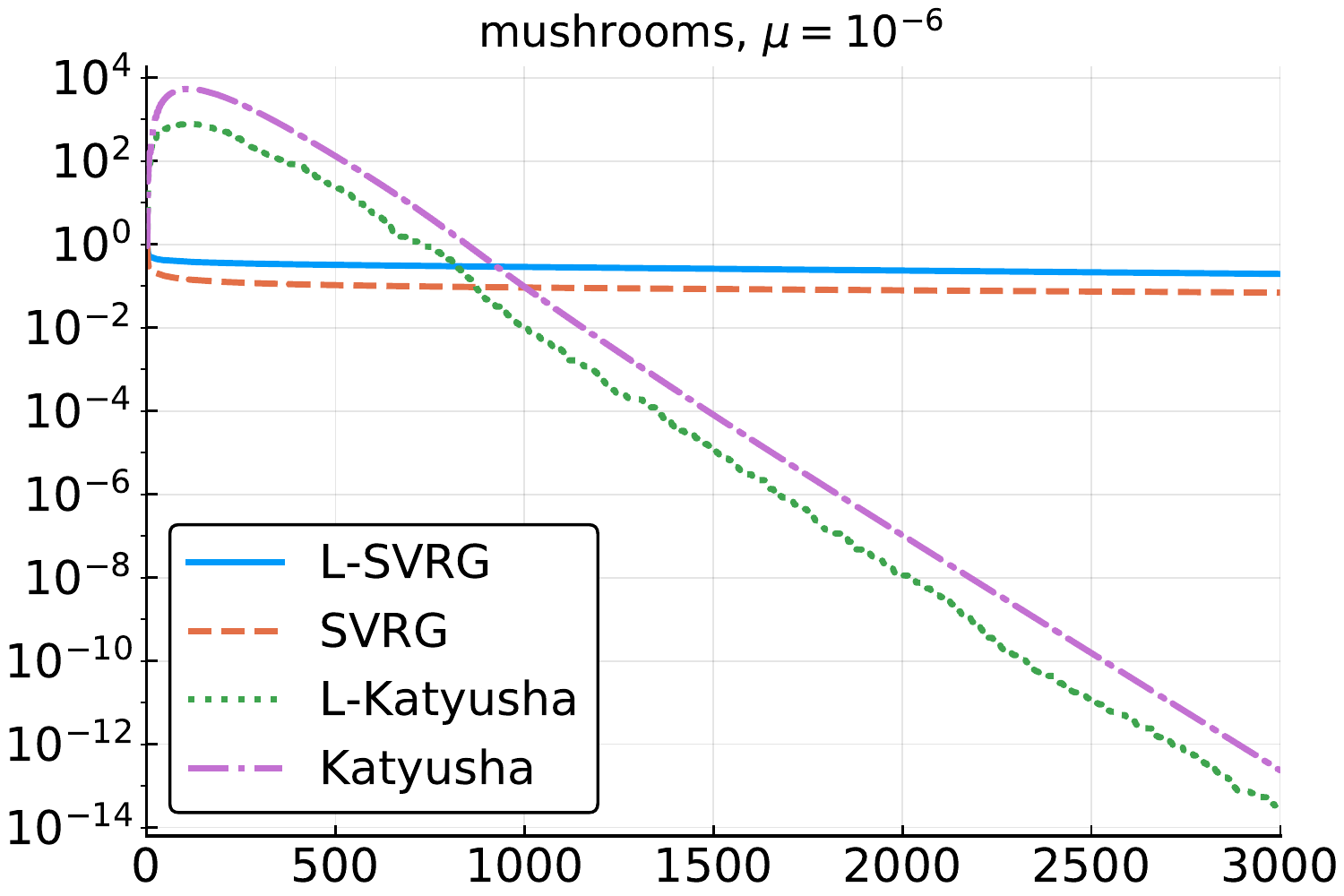}}
	
	\subfloat{\includegraphics[width=0.25\linewidth]{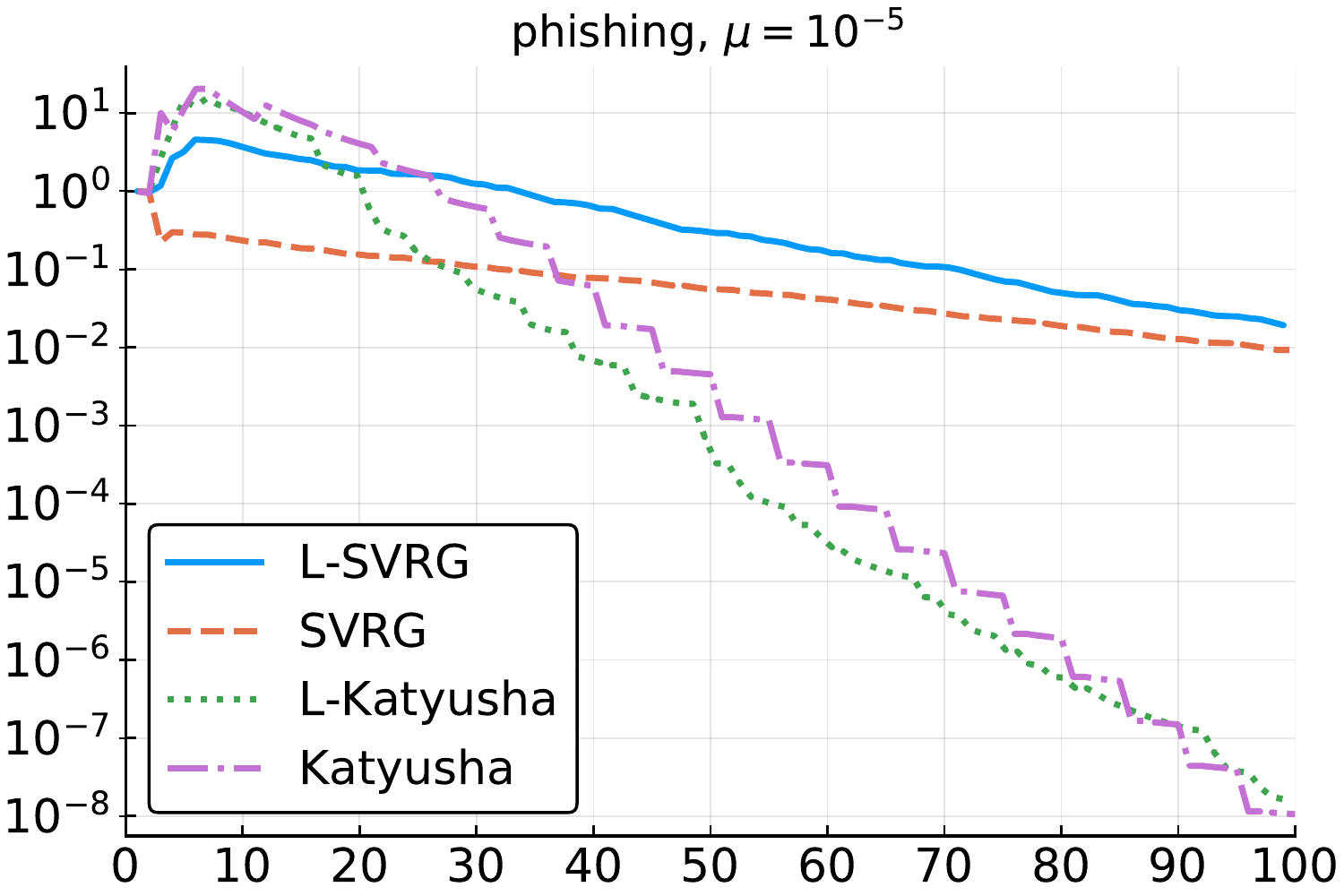}}
	\subfloat{\includegraphics[width=0.25\linewidth]{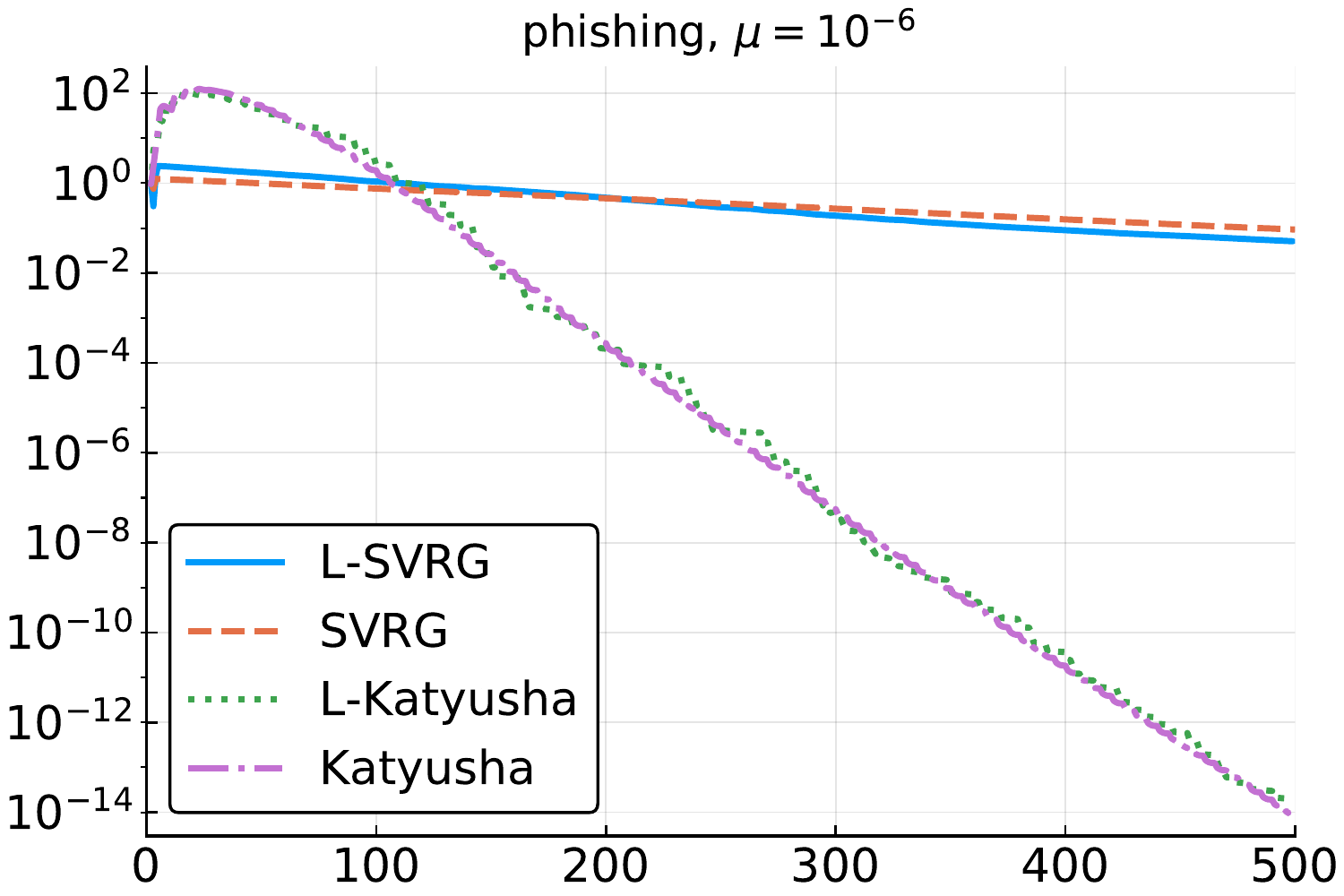}}
	\subfloat{\includegraphics[width=0.25\linewidth]{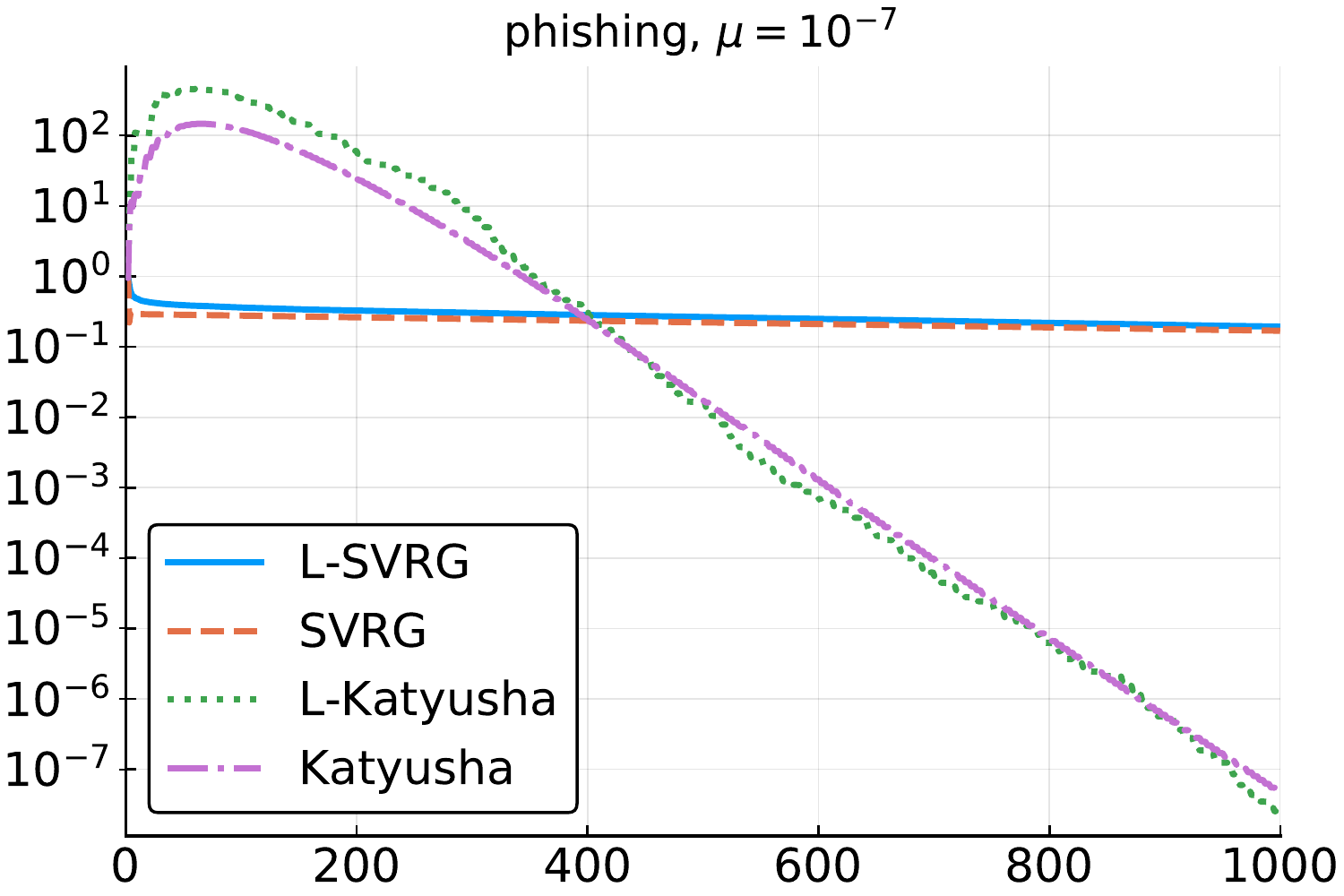}}

	\subfloat{\includegraphics[width=0.25\linewidth]{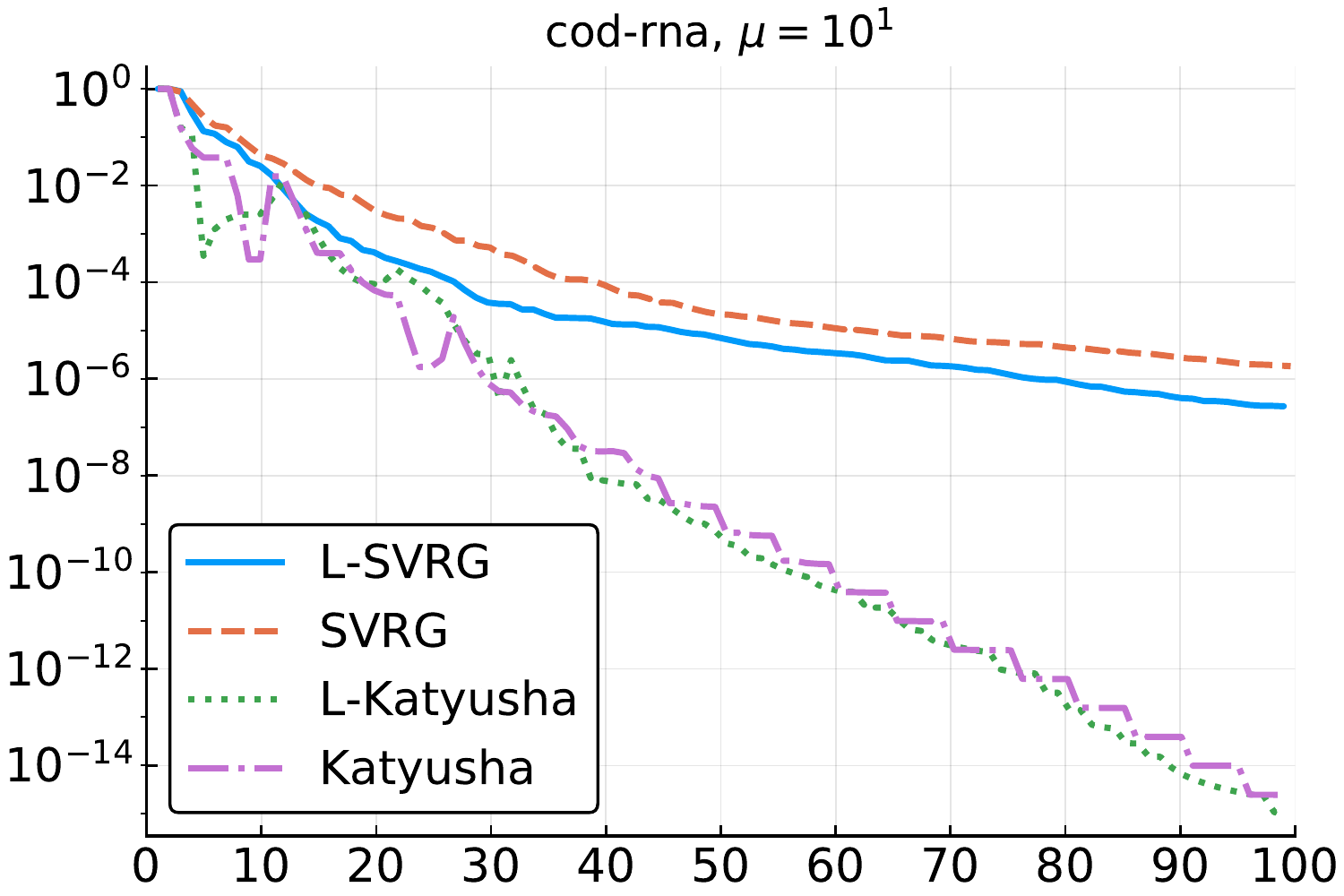}}
	\subfloat{\includegraphics[width=0.25\linewidth]{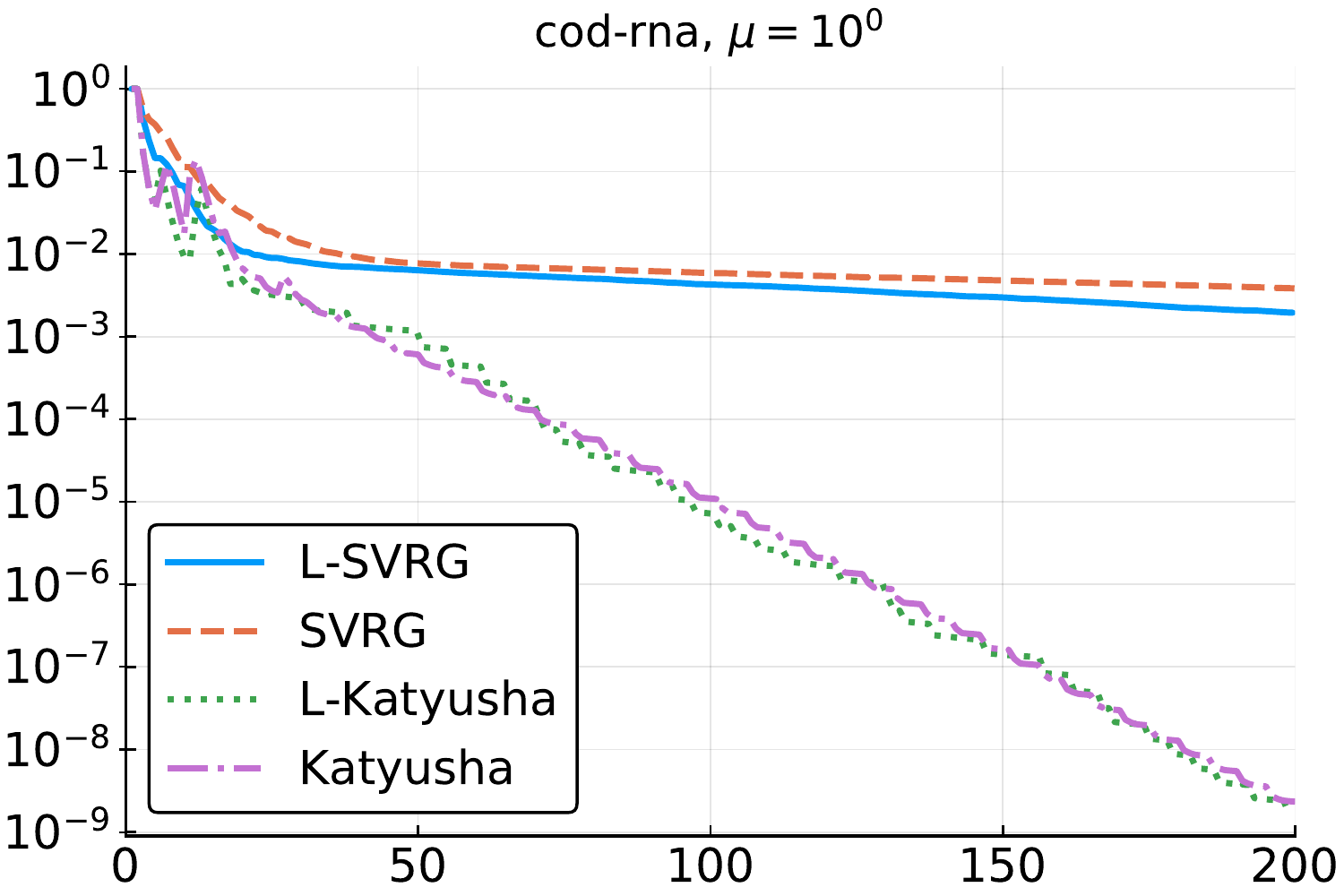}}
	\subfloat{\includegraphics[width=0.25\linewidth]{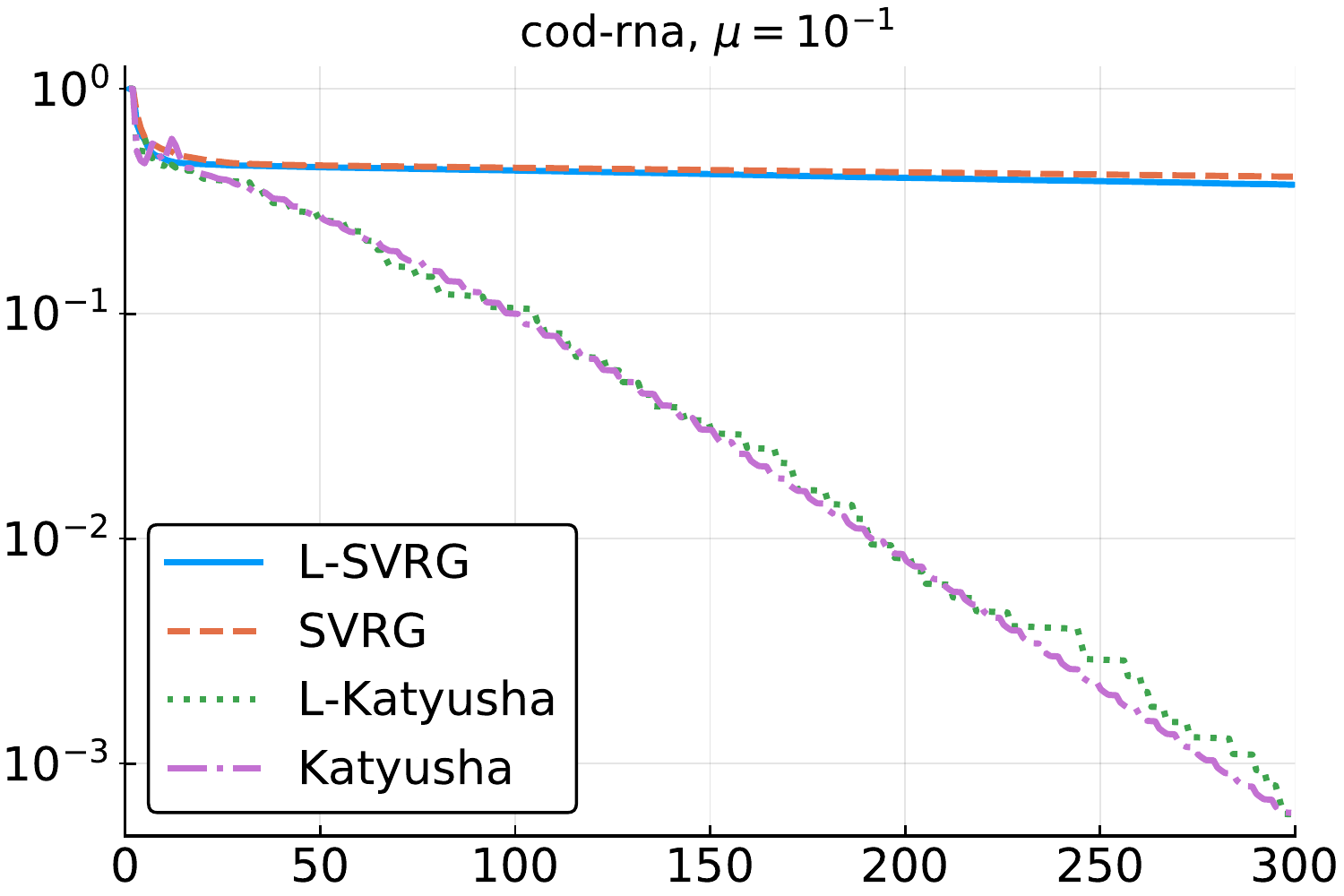}}

	\caption{All methods together for different datasets and different regularizer weights.}
	\label{fig:all} 
\end{figure*}

{\bf Different choices of probability/ outer loop size.}
We now compare several choices of the probability $p$ of updating the full gradient for \texttt{SVRG} and several outer loop sizes $m$ for \texttt{SVRG}. Since our analysis guarantees the optimal rate for any choice of $p$ between $\nicefrac{1}{n}$ and  $\nicefrac{\mu}{L}$ for well condition problems, we decided to perform experiments for $p$ within this range. More precisely, we choose $5$ values of $p$, uniformly distributed in  logarithmic scale across this interval, and thus our choices are $n$,$\sqrt[4]{\kappa n^3}$, $\sqrt{\kappa n}$, $\sqrt[4]{\kappa^3n}$,  and $\kappa$, where $\kappa = \nicefrac{L}{\mu}$, denoted in the figures by $1,2,3,4,5$, respectively. Since the expected ``outer loop'' length (length for which reference point stays the same) is $ \nicefrac{1}{p}$, for \texttt{SVRG} we choose $m = \nicefrac{1}{p}$. Looking at Figure~\ref{fig:diff_p}, one can see that \texttt{L-SVRG} is very {\em robust} to the choice of $p$ from the  ``optimal interval'' predicted by our theory.  Moreover, {\em even the worst case for \texttt{L-SVRG} outperforms the best case for \texttt{SVRG}.} 

{\bf All methods together.} Finally, we provide all algorithms together in one plot for different datasets with different regularizer weight, thus  with different condition numbers, displayed in Figure~\ref{fig:all}. As for the previous experiments, loopless methods are not worse and sometimes significantly better.


\bibliography{reference}
\bibliographystyle{plain}


\appendix

\onecolumn


\part*{Supplementary Material: \\ \Large SVRG and Katyusha are Better Without the Outer Loop}

\section{Auxiliary Lemmas}

\begin{lemma}
For random vector $x\in \R^d$ and any $y \in \R^d$, the variance of $y$ can be decomposed as
\begin{equation}
\label{eq:variance}
\E{\norm{x-\E{x}}^2} = \E{\norm{x-y}^2} - \E{\norm{\E{x}-y}^2}\,.
\end{equation}
\end{lemma}

The next lemma is a consequence of  Jensen's inequality applied to $x\mapsto \|x\|^2$.
\begin{lemma}
For any vectors $a_1, a_2, \dots, a_k \in \R^d$, the following inequality holds:
\begin{equation}
\left\|\sum_{i=1}^k a_i\right\|^2 \leq k \sum_{i=1}^k \norm{a_i}^2  \,. \label{eq:sum}
\end{equation}
\end{lemma}

\section{Proofs for Algorithm~1 (\texttt{L-SVRG})}

In all proofs below, we will for simplicity write $f^*\eqdef f(x^*)$.

\subsection{Proof of Lemma~\ref{lem:x^k_alg_1}}

Definition of $x^{k+1}$ and unbiasness of $g^k$ guarantee that
	\begin{eqnarray*}
		\E{\norm{x^{k+1} - x^*}^2} &=& \E{\norm{x^k - x^* - \eta g^k}}^2 \\
		&\overset{\text{Alg.}~\ref{alg:1}}{=}&
		\norm{x^k - x^*}^2 + \E{2 \eta\dotprod{g^k}{x^* - x^k}} + \eta^2\E{\norm{g^k}^2}\\
		&\overset{\eqref{eq:unbiased}}{=}& \norm{x^k - x^*}^2 + 2 \eta\dotprod{\nabla f(x^k)}{x^* - x^k} + \eta^2\E{\norm{g^k}^2}\\
		&\overset{\eqref{def:strong_convexity}}{\leq}& \norm{x^k - x^*}^2 +  2 \eta\left(f^* - f(x^k) - \frac{\mu}{2}\norm{x^k - x^*}\right)+ \eta^2\E{\norm{g^k}^2}\\
		&=& \norm{x^k - x^*}^2(1 - \eta\mu) +  2 \eta\left(f^* - f(x^k)\right)+ \eta^2\E{\norm{g^k}^2}.
	\end{eqnarray*}

\subsection{Proof of Lemma~\ref{lem:g^k_alg_1}}

Using definition of $g^k$
	\begin{eqnarray*}
		\E{\norm{g^k}^2} &\overset{\text{Alg.}~\ref{alg:1}}{=}&
		\E{\norm{
			\nabla f_i(x^k) - \nabla f_i(x^*) + \nabla f_i(x^*) - \nabla f_i(w^k) + \nabla f(w^k)
		}^2}\\
		&\overset{\eqref{eq:sum}}{\leq}&
		2\E{\norm{\nabla f_i(x^k) - \nabla f_i(x^*)}^2}
		+ 
		2\E{\norm{\nabla f_i(x^*) - \nabla f_i(w^k) - \E{\nabla f_i(x^*) - \nabla f_i(w^k)}}^2}\\
		&\overset{\eqref{def:L-smoothness}+\eqref{eq:variance}}{\leq}&
		4L(f(x^k) - f^*) + 2 \E{\norm{\nabla f_i(w^k) - \nabla f_i(x^*)}^2} \\
		&\overset{\eqref{def:D^k_alg_1}}{=}&
		4L(f(x^k) - f^*) + \frac{p}{2\eta^2} \cD^k.
	\end{eqnarray*}

\subsection{Proof of Lemma~\ref{lem:D^k_alg_1}}
	\begin{eqnarray*}
		\E{\cD^{k+1}} &\overset{\text{Alg.}~\ref{alg:1}}{=}& (1-p) \cD^k + p \frac{4\eta^2}{pn} \sum\limits_{i=1}^{n} \norm{\nabla f(x^k) - \nabla f(x^*)}^2\\
		&\overset{\eqref{def:L-smoothness}}{\leq}& (1-p) \cD^k + 8L\eta^2 (f(x^k) - f^*).
	\end{eqnarray*}

\subsection{Proof of Lemma~\ref{lem:conv_alg_1}}
Combining Lemmas~\ref{lem:x^k_alg_1} and \ref{lem:D^k_alg_1} we obtain
	\begin{eqnarray*}
		\E{\norm{x^{k+1} - x^*}^2 + \cD^{k+1}} &\overset{\eqref{eq:x^k_alg_1}+\eqref{eq:D^k_alg_1}}{\leq}&
		(1-\mu\eta) \norm{x^k - x^*}^2 + 2\eta(f^* - f(x^k)) + \eta^2 \E{\norm{g^k}^2}\\
		&& \quad +
		(1-p) \cD^k + 8L\eta^2 (f(x^k) - f^*)\\
		&\overset{\eqref{eq:g^k_alg_1}}{\leq}&
		(1-\mu\eta) \norm{x^k - x^*}^2 + (1-p) \cD^k +  (2\eta - 8L\eta^2)(f^* - f(x^k))\\
		&& \quad +
		\eta^2
		\left(
			4L(f(x^k) - f^*) + \frac{p}{2\eta^2} \cD^k
		\right)\\
		&= &
		(1-\mu\eta) \norm{x^k - x^*}^2 +\left(1 - \frac{p}{2}\right)  \cD^k +  (2\eta - 12L\eta^2)(f^* - f(x^k)).
	\end{eqnarray*}
	Now we use the fact that $\eta \leq \frac{1}{6L}$ and obtain the desired inequality:
	\begin{eqnarray*}
		\E{\norm{x^{k+1} - x^*}^2 + \cD^{k+1}} 
		& \leq &
		(1-\mu\eta) \norm{x^k - x^*}^2 + \left(1 - \frac{p}{2}\right)\cD^k .
	\end{eqnarray*}

\section{Proofs for Algorithm~\ref{alg:2} (\texttt{L-Katyusha})}

\subsection{Proof of Lemma~\ref{lem:g^k_alg_2}}

To upper bound the variance of $g^k$ we first uses its definition
	\begin{eqnarray*}	 
		\E{\norm{g^k - \nabla f (x^k)}^2}
		&\overset{\text{Alg.}~\ref{alg:2}}{=}& \E{\norm{\nabla f_{i}(x^k) - \nabla f_{i}(w^k) - \E{\nabla f_{i}(x^k) - \nabla f_{i}(w^k)}}^2}\\
		&\overset{\eqref{eq:variance}}{\leq}&
		\E{\norm{\nabla f_{i}(x^k) - \nabla f_{i}(w^k)}^2} \\
		&\overset{\eqref{def:L-smoothness}}{\leq}&
		2L
		\left(
		f(w^k) - f(x^k) -\dotprod{\nabla f(x^k)}{w^k - x^k}
		\right).
	\end{eqnarray*}

\subsection{Proof of Lemma~\ref{lem:wg_alg_2}}
We start with the definition of $z^{k+1}$ 
	\begin{equation*}
	z^{k+1} \overset{\text{Alg.}~\ref{alg:2}}{=} \frac{1}{1+\eta \sigma} \left(\eta\sigma x^k + z^k - \frac{\eta}{L}g^k\right),
	\end{equation*}
	which implies
$
		\frac{\eta}{L}g^k = \eta\sigma(x^k - z^{k+1}) + (z^k - z^{k+1}),
$
	which further implies that 
	\begin{eqnarray*}
		\dotprod{g^k}{z^{k+1} - x^*} &= &
		\mu\dotprod{x^k - z^{k+1}}{z^{k+1} - x^*}
		+
		\frac{L}{\eta}\dotprod{z^k - z^{k+1}}{z^{k+1} - x^*}\\
		&
		= &
		\frac{\mu}{2}\left(
			\norm{x^k - x^*}^2 - \norm{x^k - z^{k+1}}^2 - \norm{z^{k+1} - x^*}^2
		\right)
		\\
		&& \quad +
		\frac{L}{2\eta}\left(
			\norm{z^k - x^*}^2 - \norm{z^k - z^{k+1}}^2 - \norm{z^{k+1} - x^*}^2
		\right)\\
		&		\leq &
		\frac{\mu}{2}\norm{x^k - x^*}^2
		+
		\frac{L}{2\eta}
		\left(
			\norm{z^k - x^*}^2
			-
			(1+\eta\sigma)\norm{z^{k+1} - x^*}^2
		\right)
		-
		\frac{L}{2\eta}
		\norm{z^k - z^{k+1}}^2.
	\end{eqnarray*}

\subsection{Proof of Lemma~\ref{lem:wtheta_alg_2}}
	\begin{eqnarray*}
		&& \frac{L}{2\eta} \norm{z^{k+1} - z^k}^2
		+
		\dotprod{g^k}{z^{k+1} - z^k} \\
		&	= &
		\frac{1}{\theta_1}
		\left(
			\frac{L}{2\eta\theta_1}\norm{\theta_1(z^{k+1} - z^k)}^2
			+
			\dotprod{g^k}{\theta_1(z^{k+1} - z^k)}
		\right)\\
		&\overset{\text{Alg.}~\ref{alg:2}}{=}&
		\frac{1}{\theta_1}
		\left(
			\frac{L}{2\eta\theta_1}\norm{y^{k+1} - x^k}^2
			+
			\dotprod{g^k}{y^{k+1} - x^k}
		\right)\\
		&=&
		\frac{1}{\theta_1}
		\left(
			\frac{L}{2\eta\theta_1}\norm{y^{k+1} - x^k}^2
			+
			\dotprod{\nabla f(x^k)}{y^{k+1} - x^k}
			+
			\dotprod{g^k-\nabla f(x^k)}{y^{k+1} - x^k}
		\right)\\
		&=&
		\frac{1}{\theta_1}
		\left(
			\frac{L}{2}\norm{y^{k+1} - x^k}^2
			+
			\dotprod{\nabla f(x^k)}{y^{k+1} - x^k}
			+
			\frac{L}{2} \left(
				\frac{1}{\eta\theta_1} - 1
			\right)
			\norm{y^{k+1} - x^k}^2
			+
			\dotprod{g^k-\nabla f(x^k)}{y^{k+1} - x^k}
		\right)\\
		&\overset{\eqref{def:L-smoothness}}{\geq} &
		\frac{1}{\theta_1}
		\left(
			f(y^{k+1}) - f(x^k)
			+
			\frac{L}{2} \left(
				\frac{1}{\eta\theta_1} - 1
			\right)
			\norm{y^{k+1} - x^k}^2
			+
			\dotprod{g^k-\nabla f(x^k)}{y^{k+1} - x^k}
		\right)\\
		&\geq &
		\frac{1}{\theta_1}
		\left(
			f(y^{k+1}) - f(x^k)
			-
			\frac{\eta\theta_1}{2L(1-\eta\theta_1)}
			\norm{g^k - \nabla f(x^k)}^2
		\right) \\
		&= &
		\frac{1}{\theta_1}
		\left(
			f(y^{k+1}) - f(x^k)
			-
			\frac{\theta_2}{2L}
			\norm{g^k - \nabla f(x^k)}^2
		\right),
	\end{eqnarray*}
	where the last inequality uses the Young's inequality in the form of $\dotprod{a}{b}\geq -\frac{\norm{a}^2}{2\beta}  -\frac{\beta\norm{b}^2}{2}$ for $\beta = \frac{\eta\theta_1}{L(1-\eta\theta_1)}$, which concludes the proof.

\subsection{Proof of Lemma~\ref{lem:W^k_alg_2}}
	From the definition of $w^{k+1}$ in Algorithm~\ref{alg:2} we have
	\begin{equation}
		\E{f(w^{k+1})} \overset{\text{Alg.}~\ref{alg:2}}{=} (1-p)f(w^k) + p f(y^k).
	\end{equation}
	 The rest of proof follows from the definition of $\cW^k$ \eqref{eq:W^k_alg_2}.

\clearpage

\subsection{Proof of Lemma~\ref{lem:conv_alg_2}}
	Combining all the previous lemmas together, we obtain
	\begin{eqnarray*}
		f^* &\overset{\eqref{def:strong_convexity}}{\geq}&
		f(x^k) + \dotprod{\nabla f(x^k)}{x^* - x^k} + \frac{\mu}{2}\norm{x^k - x^*}^2\\
		&=&
		f(x^k) + \frac{\mu}{2}\norm{x^k - x^*}^2 +  \dotprod{\nabla f(x^k)}{x^* -z^k + z^k - x^k}\\
		&\overset{\text{Alg.}~\ref{alg:2}}{=}&
		f(x^k) + \frac{\mu}{2}\norm{x^k - x^*}^2 +  \dotprod{\nabla f(x^k)}{x^* -z^k}
		+
		\frac{\theta_2}{\theta_1}\dotprod{\nabla f(x^k)}{x^k-w^k} + \frac{(1-\theta_1 - \theta_2)}{\theta_1}\dotprod{\nabla f(x^k)}{x^k - y^k}\\
		&\overset{ \eqref{eq:unbiased}}{\geq}&
		f(x^k)
		+
		\frac{\theta_2}{\theta_1}\dotprod{\nabla f(x^k)}{x^k-w^k} + \frac{(1-\theta_1 - \theta_2)}{\theta_1}(f(x^k) - f(y^k))\\
		&& \quad +
		\E{
			\frac{\mu}{2}\norm{x^k - x^*}^2 +  \dotprod{g^k}{x^* -z^{k+1}} + \dotprod{g^k}{z^{k+1} - z^k}
		}\\
		&\overset{\eqref{eq:wg_alg_2}}{\geq}&
		f(x^k)
		+
		\frac{\theta_2}{\theta_1}\dotprod{\nabla f(x^k)}{x^k-w^k} + \frac{(1-\theta_1 - \theta_2)}{\theta_1}(f(x^k) - f(y^k))\\
		&& \quad +
		\E{
			\cZ^{k+1}
			-
			\frac{1}{1+\eta\sigma}\cZ^k
		}
		+
		\E{
			\dotprod{g^k}{z^{k+1} - z^k}
			+
			\frac{L}{2\eta}\norm{z^k - z^{k+1}}^2
		}\\
		&\overset{\eqref{eq:wtheta_alg_2}}{\geq}&
		f(x^k)
		+
		\frac{\theta_2}{\theta_1}\dotprod{\nabla f(x^k)}{x^k-w^k} + \frac{(1-\theta_1 - \theta_2)}{\theta_1}(f(x^k) - f(y^k))\\
		& &\quad +
		\E{
			\cZ^{k+1}
			-
			\frac{1}{1+\eta\sigma}\cZ^k
		}
		+
		\E{
			\frac{1}{\theta_1}\left(f(y^{k+1}) - f(x^k)\right)-\frac{\theta_2}{2L\theta_1} \norm{g^k - \nabla f(x^k)}^2
		}\\
		&\overset{\eqref{eq:g^k_alg_2}}{\geq}&
		f(x^k)
		+
		\frac{\theta_2}{\theta_1}\dotprod{\nabla f(x^k)}{x^k-w^k} + \frac{(1-\theta_1 - \theta_2)}{\theta_1}(f(x^k) - f(y^k))\\
		&& \quad +
		\E{
			\cZ^{k+1}
			-
			\frac{1}{1+\eta\sigma}\cZ^k
		}
		+
		\E{
			\frac{1}{\theta_1}\left(f(y^{k+1}) - f(x^k)\right)
			-
			\frac{\theta_2}{\theta_1}
			\left(
				f(w^k) - f(x^k) - \dotprod{\nabla f(x^k)}{w^k - x^k}
			\right)
		}\\
	&=&
	f(x^k)
	+
	\frac{(1-\theta_1 - \theta_2)}{\theta_1}(f(x^k) - f(y^k)) - \frac{1}{1+\eta\sigma}\cZ^k
	-
	\frac{\theta_2}{\theta_1}(f(w^k) - f(x^k))
	\\
	&& \quad +
	\E{
		\cZ^{k+1} 
		+
		\frac{1}{\theta_1}\left(f(y^{k+1}) - f(x^k)\right)
	},
	\end{eqnarray*}
	where in the second inequality we use also convexity of $f(x)$.
	\begin{eqnarray*}
		x^k &\overset{\text{Alg.}~\ref{alg:2}}{=}& \theta_1 z^k + \theta_2 w^k + (1-\theta_1 - \theta_2) y^k\\
		z^k - x^k &=& \frac{\theta_2}{\theta_1}(x^k-w^k) + \frac{1-\theta_1 - \theta_2}{\theta_1}(x^k - y^k).
	\end{eqnarray*}

	After rearranging we get
	\begin{eqnarray*}
		\frac{1}{1+\eta\sigma}\cZ^k + (1-\theta_1 - \theta_2)\cY^k + \frac{\theta_2}{\theta_1}(f(w^k - f^*))\geq
		\E{
			\cZ^{k+1} + \cY^{k+1}
		}.
	\end{eqnarray*}
	Using definition of $\cW^k$ we get
	\begin{equation}
	\E{
		\cZ^{k+1} + \cY^{k+1}
	}
	\leq
	\frac{1}{1+\eta\sigma}\cZ^k + (1-\theta_1 - \theta_2)\cY^k +	\frac{p}{(1+\theta_1)}\cW^k.
	\end{equation}
	
	Finally, using Lemma~\ref{lem:W^k_alg_2} we get
	\begin{eqnarray*}
		\E{
			\cZ^{k+1}
			+
			\cY^{k+1}
			+
			\cW^{k+1}
		}
		&\leq &
			\frac{1}{1+\eta\sigma}\cZ^k + (1-\theta_1 - \theta_2)\cY^k +	\frac{p}{(1+\theta_1)}\cW^k
			+
			(1-p)\cW^k + \theta_2(1+\theta_1) \cY^k\\
		&= &
			\frac{1}{1+\eta\sigma}\cZ^k + (1-\theta_1(1-\theta_2))\cY^k
			+
			\left(
				1 - \frac{p\theta_1}{1+\theta_1}
			\right)\cW^k,
	\end{eqnarray*}
	which concludes the proof.

\end{document}